\documentclass[11pt]{article}
\usepackage{amsmath,amssymb}
\usepackage{amsthm}
\theoremstyle{plain}
\newtheorem{proposition}{Proposition}[section]
\newtheorem{corollary}[proposition]{Corollary}
\usepackage{graphicx}
\usepackage[table]{xcolor}
\usepackage{booktabs}
\usepackage{algorithm}
\usepackage{algpseudocode}
\usepackage[numbers,sort&compress]{natbib}
\usepackage{etoolbox}
\usepackage{array}
\usepackage{adjustbox}
\usepackage{placeins}
\usepackage{xurl}
\usepackage[hidelinks,breaklinks]{hyperref}

\makeatletter
\newcommand{\keywords}[1]{\par\medskip\noindent\textbf{Keywords.}\ #1\par}
\newcommand{\subclass}[1]{\par\medskip\noindent\textbf{Mathematics Subject Classification (2020).}\ #1\par}
\makeatother
\newcommand{\doi}[1]{\url{https://doi.org/#1}}

\hypersetup{pdftitle={DeSyR: A decoupled symbolic recovery framework for PDEs with PINN-guided structure discovery and physics-based coefficient refinement},pdfauthor={Pancheng Niu, Jun Guo, Qiaolin He, Jingcai Guo, Yanchao Shi}}
\title{DeSyR: A decoupled symbolic recovery framework for PDEs with PINN-guided structure discovery and physics-based coefficient refinement}
\author{Pancheng Niu\thanks{College of Applied Mathematics, Chengdu University of Information Technology, Chengdu 610225, China.} \and Jun Guo\thanks{College of Applied Mathematics, Chengdu University of Information Technology, Chengdu 610225, China. Corresponding author: \mbox{\href{mailto:junguo0407@cuit.edu.cn}{junguo0407@cuit.edu.cn}}.} \and Qiaolin He\thanks{School of Mathematics, Sichuan University, Chengdu 610065, China.} \and Jingcai Guo\thanks{Department of Computing, Hong Kong Polytechnic University, Hong Kong 999077, China.} \and Yanchao Shi\thanks{College of Science, Southwest Petroleum University, Chengdu 610500, Sichuan, China.}}
\date{}
\begin{document}
\maketitle
\begin{abstract}
Recovering explicit symbolic solutions of partial differential equations
(PDEs) from physics-informed neural networks (PINNs) is attractive, yet
selecting and fitting expressions against neural predictions can lock the
recovered solution into the teacher's approximation error. We propose
DeSyR (Decoupled Symbolic Recovery) to break this dependence. DeSyR uses
PINNs only to discover candidate structures and initial coefficient
estimates; once a structure is fixed, the governing equations and
constraints alone determine the refined coefficients. Multiple candidates
with different fitting-error--complexity trade-offs are retained, evaluated
after physics-only refinement, and reduced to the simplest expression within
a physically near-optimal set. Joint refinement extends the same principle
to coupled fields. For fixed candidate parameterizations, affine error
identities and local nonlinear analysis show how teacher errors enter fitted
coefficients and establish conditions under which physical residuals control
coefficient and solution errors. Across seven benchmarks covering scalar
evolution equations, multidimensional elliptic problems and coupled
incompressible flows, DeSyR attains mean relative solution errors on the
order of $10^{-14}$ or smaller and lower reported mean physical residuals
than PINNs, physics-informed symbolic networks and teacher-fitted symbolic
distillation. With the expression structure held fixed, estimating
coefficients from physical constraints further yields more accurate
frequency estimates and solutions than fitting teacher predictions. DeSyR
therefore turns neural approximations into guides for structure discovery,
while physical constraints determine the final explicit solution.
\subclass{68W30 \and 68T07}
\end{abstract}

\section{Introduction}
\label{sec:introduction}
Partial differential equations (PDEs) describe a wide range of physical
processes, including diffusion, wave propagation and fluid flow.
Studying these processes requires determining solutions that satisfy the
governing equations and the associated initial and boundary conditions.
Because explicit solutions are available only for particular classes of
problems, numerical methods are widely used to obtain approximate solutions.
Finite difference, finite element and spectral methods have been developed
for this purpose and provide established tools for solving PDEs across
diverse equations and domains
\cite{leveque2007finite,brenner2008mathematical,trefethen2000spectral}.

With advances in deep learning, neural networks have also been used to
approximate PDE solutions, representing solution fields as continuous,
trainable functions. The Deep Galerkin and Deep Ritz methods train these
functions using differential residuals and variational objectives,
respectively \cite{sirignano2018dgm,yu2018deep}.
Physics-informed neural networks (PINNs) incorporate governing equations
and prescribed constraints into the training loss and can be trained
without labeled interior solution values
\cite{raissi2019physics,karniadakis2021physics}.
Once trained, a PINN can be sampled throughout the domain to provide
approximate solution data, although its accuracy may be limited by gradient
imbalances and optimization difficulties
\cite{wang2021understanding,krishnapriyan2021characterizing,wang2022and,niu2025improved}.
Beyond approximation accuracy, the form in which the solution is represented
also matters: a neural network does not directly expose a compact symbolic
expression. When such an expression exists, it can reveal the dependence of
the solution on its variables and parameters in a form accessible to direct
analysis.

The prospect of obtaining compact expressions from numerical approximations
motivates the use of symbolic regression. By searching over combinations of
variables, constants and mathematical operators
\cite{schmidt2009distilling}, symbolic regression determines expression
structures and coefficients while balancing fitting accuracy against
complexity \cite{la2021contemporary,cranmer2023interpretable}.
A trained neural model can supply samples or intermediate representations
for this search. For example, Cranmer et al.\ extracted symbolic relations
from components of trained graph neural networks
\cite{cranmer2020discovering}.
For PDE solution recovery, however, neural approximations are not the only
available source of information: the governing equations and constraints
can also guide the search and assess the resulting expressions. This has
led to approaches that differ in how they combine numerical solution data
with physical relations.

One approach is to use the governing equations and constraints directly to
search for symbolic solutions, without first training a neural solver.
Oh et al.\ proposed physics-regularized genetic programming symbolic
regression, which combines expression evolution with local
coefficient optimization based on differential-equation and
boundary-condition residuals \cite{oh2024inherently}.
The physics-informed symbolic network (PISN) approach instead represents
expressions through a differentiable network of symbolic operators and
trains it with a physics-informed loss \cite{majumdar2022physics}.
Similarly, the Symbolic Equation Solver (SES) optimizes a differentiable
symbolic model using equation and constraint residuals without paired
solution observations \cite{garmaev2026data}.
In these methods, physical relations provide the basis for constructing
and evaluating symbolic solutions.

An alternative is to begin with a trained neural approximation and recover
an explicit expression through symbolic distillation. Majumdar et al.\
fitted differentiable programs to PINN-generated samples and pruned them
to obtain sparser expressions \cite{majumdar2023symbolic}.
Das et al.\ applied genetic-programming symbolic regression to PINN
predictions for nonlinear evolution equations \cite{das2025physics},
while Changdar et al.\ combined improved PINN training with symbolic
extraction of approximate analytical solutions \cite{changdar2026refined}.
SymTorch provides a general interface for symbolic distillation of neural
components, including PINN-based recovery of differential-equation
solutions \cite{tan2026symtorch}.
Here, the neural approximation serves as a teacher that supplies
function-value targets for expression fitting. This makes its predicted
solution available to symbolic search, but a close fit to those predictions
does not by itself ensure satisfaction of the original equations and
constraints.

Beyond using neural predictions as fitting targets, researchers have also
explored structural information within trained PINNs to guide expression
search. StruSR compares candidate expressions with local Taylor
coefficients extracted from a trained PINN and uses structural and
physical sensitivities to guide expression evolution \cite{gong2025strusr}.
Residual and Structural Sensitivity Pruning (RSSP) further develops
physics-informed sensitivity-based pruning \cite{gong2026symbolic}.
These studies illustrate another role for a neural teacher: it can provide
structural descriptors that inform the search, rather than only function
values to be reproduced. This distinction is relevant to how teacher
information is used in discovering expression structures and estimating
their coefficients.

Taken together, these studies show that neural approximations and physical
relations can play different roles in symbolic solution recovery. A neural
teacher can guide the discovery of candidate structures, while the
governing equations and constraints provide a basis for determining whether
the resulting expressions solve the original problem. Using the teacher
for structural guidance, however, does not require its predictions to
remain the fitting targets for the final coefficients. Even when a
candidate has an appropriate structure, coefficients fitted to approximate
teacher values may retain teacher-induced errors. Likewise, a candidate
preferred by teacher fit need not have the smallest physical residuals
after coefficient refinement. This motivates a framework that preserves
teacher guidance during candidate discovery but uses physical constraints
to refine coefficients and select the final expression. The corresponding
analysis asks how teacher errors propagate under teacher-fitting and mixed
objectives, and under what conditions small physical residuals control
coefficient and solution errors within a fixed candidate parameterization.

To address these issues, we propose DeSyR (Decoupled Symbolic Recovery), a
framework that separates teacher-guided expression discovery from
physics-based coefficient refinement and final selection. The neural
teacher guides the search for candidate structures and provides initial
coefficient estimates, while the governing equations and constraints
determine the refinement objective and the criterion for selecting the
final expression. This separation allows DeSyR to exploit structural
information in neural approximations without requiring the refined
expressions to fit teacher predictions. We support this design with an
analysis of teacher-error propagation and conditions for stable
coefficient recovery, together with numerical comparisons that hold the
expression structure fixed.

The main contributions are as follows:
\begin{enumerate}
\item \textbf{A decoupled framework for symbolic solution recovery.}
We integrate multi-candidate retention, physics-only coefficient
refinement and physics-based model selection. Retaining candidates with
different fitting-error--complexity trade-offs allows alternative
structures to be assessed after refinement, rather than committing to
the teacher-fit ranking. The final expression is selected by minimum
complexity within a physically near-optimal set. For coupled fields,
system candidates are constructed and refined jointly under the complete
equations and constraints.
\item \textbf{An analysis of coefficient-estimation error within fixed
candidate parameterizations.}
We derive affine error identities and local nonlinear perturbation
results that characterize how teacher errors enter coefficient estimates
under data-fitting and mixed objectives. We also establish identifiability
conditions and stability bounds for physics-based coefficient recovery,
relating physical residuals to coefficient and solution errors under
stated representability, rank and locality assumptions.
\item \textbf{Numerical evidence for recovery accuracy and the role of the
coefficient objective.}
We evaluate DeSyR against PINNs, PISN and a teacher-fitted symbolic
regression baseline (PINN--SR) on seven problems covering scalar evolution
equations, multidimensional elliptic equations and coupled incompressible
flows. DeSyR and PINN--SR share the same teachers and symbolic searches,
allowing the complete recovery procedures to be compared through solution
errors, physical residuals and recovered expressions. A separate
fixed-candidate study shows that, within the examined expression family,
estimating coefficients from physical constraints yields more accurate
coefficients and solutions than fitting teacher predictions.
\end{enumerate}

The remainder of this paper is organized as follows.
Section~\ref{sec:preliminaries} formulates the symbolic solution recovery
problem and introduces the PINN and symbolic regression foundations.
Section~\ref{sec:method} presents the DeSyR framework and its extension
to coupled fields. Section~\ref{sec:theory} analyzes teacher-error
propagation and establishes conditions for identifiable and stable
coefficient recovery under physical constraints.
Section~\ref{sec:experiments} reports numerical results on seven PDE
problems and examines the effect of the coefficient-estimation objective
through a fixed-candidate comparison. Section~\ref{sec:conclusion}
summarizes the main findings and outlines future research directions.

\section{Preliminaries}
\label{sec:preliminaries}

We first specify the governing equations and constraints for the PDE
problems considered in this work. We then introduce neural and symbolic
representations of their solutions, together with the objectives used to
determine them.

\subsection{Problem formulation}
\label{sec:problem}

We consider PDE problems in which the governing equations, physical
parameters and initial and boundary conditions are prescribed. To describe
both scalar equations and coupled systems, let
$\boldsymbol{u}:\Omega\to\mathbb{R}^{m}$ denote the unknown solution field,
where $m$ is the number of solution components and
$\Omega\subset\mathbb{R}^{d}$. The coordinate vector
$\boldsymbol{x}\in\Omega$ contains spatial variables for steady problems
and both spatial and temporal variables for evolution problems. The
governing equations and constraints are written as
\begin{alignat}{2}
\mathcal{N}[\boldsymbol{u}](\boldsymbol{x})
&=\boldsymbol{f}(\boldsymbol{x}),
&\qquad &\boldsymbol{x}\in\Omega,
\label{eq:governing-problem}\\
\mathcal{B}_{\ell}[\boldsymbol{u}](\boldsymbol{x})
&=\boldsymbol{g}_{\ell}(\boldsymbol{x}),
&\qquad &\boldsymbol{x}\in\Gamma_{\ell},
\label{eq:prescribed-constraints}
\end{alignat}
Here, $\mathcal{N}$ is a prescribed linear or nonlinear differential
operator, and $\boldsymbol{f}$ is a known source term. Each operator
$\mathcal{B}_{\ell}$ specifies function values or derivatives through
prescribed data $\boldsymbol{g}_{\ell}$ on
$\Gamma_{\ell}\subseteq\overline{\Omega}$, $\ell=1,\ldots,L$.
These sets may include spatial boundaries and, for evolution problems,
the initial time slice. The solution is assumed sufficiently regular for
the differential operators and the prescribed constraints to be well defined.

\subsection{Physics-informed neural networks}
\label{sec:pinn}

For the PDE problem defined above, PINNs approximate the unknown solution
field with a neural network and determine its parameters by minimizing
equation and constraint residuals~\cite{raissi2019physics}.
Let $\boldsymbol{u}_{\theta}(\boldsymbol{x})$ denote a neural approximation
with trainable parameters $\theta$ and sufficient differentiability for
the required operators. Substituting $\boldsymbol{u}_{\theta}$ into the
governing equations and constraints yields
\begin{align}
\boldsymbol{r}_{\theta}(\boldsymbol{x})
&=\mathcal{N}[\boldsymbol{u}_{\theta}](\boldsymbol{x})
-\boldsymbol{f}(\boldsymbol{x}),
\label{eq:pinn-equation-residual}\\
\boldsymbol{c}_{\ell,\theta}(\boldsymbol{x})
&=\mathcal{B}_{\ell}[\boldsymbol{u}_{\theta}](\boldsymbol{x})
-\boldsymbol{g}_{\ell}(\boldsymbol{x}).
\label{eq:pinn-constraint-residual}
\end{align}
The required derivatives with respect to the input coordinates are
evaluated by automatic differentiation~\cite{baydin2018automatic}.

The training objective combines the mean squared equation and constraint
residuals:
\begin{equation}
\mathcal{L}_{\mathrm{PINN}}(\theta)
=\frac{\lambda_r}{N_r}\sum_{i=1}^{N_r}
\bigl\|\boldsymbol{r}_{\theta}(\boldsymbol{x}_{i}^{r})\bigr\|_2^2
+\sum_{\ell=1}^{L}\frac{\lambda_\ell}{N_\ell}
\sum_{j=1}^{N_\ell}
\bigl\|\boldsymbol{c}_{\ell,\theta}
(\boldsymbol{x}_{j}^{c,\ell})\bigr\|_2^2.
\label{eq:pinn-loss}
\end{equation}
Here, $\{\boldsymbol{x}_{i}^{r}\}_{i=1}^{N_r}\subset\Omega$ are interior
collocation points, and
$\{\boldsymbol{x}_{j}^{c,\ell}\}_{j=1}^{N_\ell}\subset\Gamma_\ell$ are
points for constraint group $\ell$, with respective sample counts $N_r$
and $N_\ell$. The norm $\|\cdot\|_2$ is the Euclidean norm, and the
positive weights $\lambda_r$ and $\lambda_\ell$ balance the equation and
constraint terms. This formulation imposes the constraints through
penalty terms. Gradient-based optimizers, such as
Adam~\cite{kingma2015adam} and L-BFGS~\cite{liu1989limited}, are used to
minimize this objective, yielding a continuous neural approximation that
can be evaluated and differentiated throughout the domain.

\subsection{Symbolic regression}
\label{sec:symbolic-regression}

Whereas a neural network represents the solution through its trainable
parameters, symbolic regression seeks an explicit formula by identifying
both its structure and numerical coefficients~\cite{la2021contemporary}.
Candidate formulas are assembled from a prescribed operator library
$\mathcal{O}$ and represented as expression trees subject to size or depth
restrictions. The leaves are input variables or numerical constants, and the
internal nodes are operators. Unary operators, such as $\sin$ and $\exp$,
act on one argument; binary operators, such as addition and multiplication,
act on two. A candidate is denoted by
$s_{\mathcal{T}}(\boldsymbol{x};\boldsymbol{a})$, where $\mathcal{T}$
specifies the tree topology, the operator and variable labels and the
locations of constants, and $\boldsymbol{a}$ contains the corresponding
numerical values.

For a given dataset, candidate expressions are assessed by how closely
their predictions match the supplied output values. Given the scalar dataset
$\mathcal{D}=\{(\boldsymbol{x}_i,y_i)\}_{i=1}^{N}$, with
$\boldsymbol{x}_i\in\mathbb{R}^{d}$ and $y_i\in\mathbb{R}$, the mean
squared fitting loss is defined as
\begin{equation}
\mathcal{L}_{\mathrm{data}}(\mathcal{T},\boldsymbol{a})
=\frac{1}{N}\sum_{i=1}^{N}
\left[s_{\mathcal{T}}(\boldsymbol{x}_i;\boldsymbol{a})-y_i\right]^2.
\label{eq:sr-data-loss}
\end{equation}
Fitting loss alone does not account for the complexity of an expression.
Symbolic regression therefore also considers a complexity measure
$C(\mathcal{T})$, commonly defined by the number of tree nodes or the total
cost assigned to operators and terminals. The trade-off between loss and
complexity is described by Pareto ordering: one candidate dominates another if its loss and complexity are
both no greater, with at least one being strictly smaller. The nondominated
candidates encountered during the search form an empirical Pareto
front~\cite{la2021contemporary,cranmer2023interpretable}.

To search over these candidate structures, genetic programming uses
population-based selection, mutation and crossover.
PySR~\cite{cranmer2023interpretable} combines this evolutionary search
with algebraic simplification and numerical optimization of constants.

\section{The DeSyR framework}
\label{sec:method}

DeSyR recovers explicit symbolic solutions from PINN approximations through
teacher-guided candidate discovery and physics-based coefficient refinement.
The discovery stage provides candidate structures and initial coefficient
estimates. The refinement stage holds each structure and its parameterization
fixed and optimizes the free coefficients using only the governing equations
and constraints. Refined candidates are then compared by physical residuals,
with the simplest expression selected from a physically near-optimal set.
For coupled systems, candidate fields are assembled and refined jointly
under the complete equations and constraints.
Figure~\ref{fig:desyr-framework} summarizes these stages and their interfaces.

\begin{figure}[htbp]
\centering
\includegraphics[width=\linewidth]{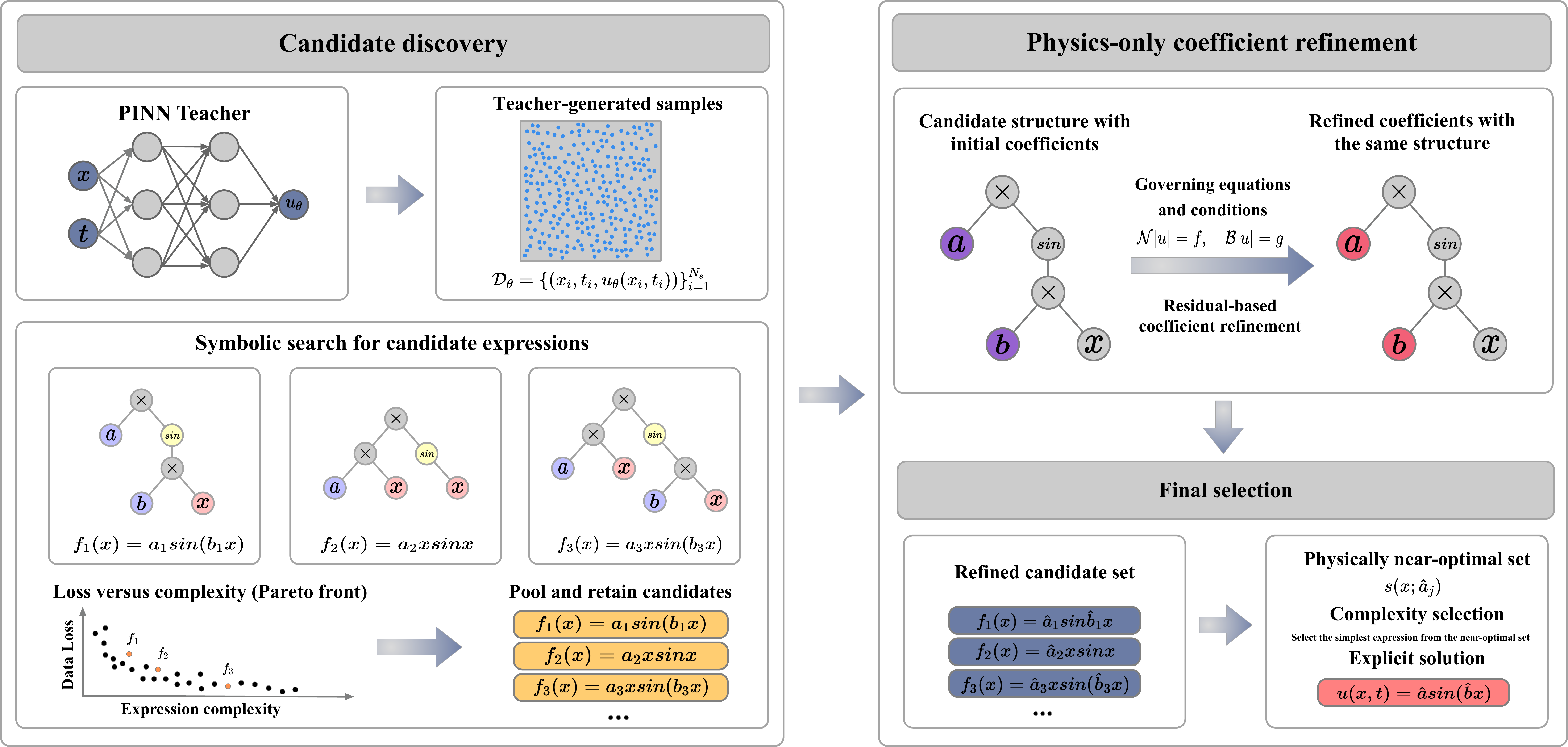}
\caption{Overview of DeSyR, separating PINN-guided candidate discovery
from physics-based coefficient refinement and model selection.
Multiple candidate structures and initial coefficient estimates are
obtained from teacher predictions. With the structures fixed,
coefficients are refined using only the governing equations and
constraints, and the simplest expression in the physically near-optimal
set is selected}
\label{fig:desyr-framework}
\end{figure}

\subsection{PINN-guided symbolic structure discovery}
\label{sec:structure-discovery}

A trained PINN provides approximate solution samples for symbolic regression,
but agreement with these samples does not determine a candidate's physical
quality after coefficient refinement. We therefore retain candidates with
different fitting-error--complexity trade-offs, together with their fitted
coefficients, so that alternative structures remain available for subsequent
physical assessment.

\subsubsection{PINN teacher and sample generation}
\label{sec:teacher-sampling}

Let $\boldsymbol{u}_{\theta_T}$ denote the trained PINN teacher, whose
parameters $\theta_T$ remain fixed throughout symbolic recovery.
Evaluating the teacher at $N_T$ query points chosen independently of the
training points provides approximate solution values for symbolic search.
These teacher samples, together with prescribed function-value constraint
data when included for scalar problems, form the regression dataset
\begin{equation}
\mathcal{D}_T
=\left\{\left(\boldsymbol{x}_i^T,
\boldsymbol{y}_i^T\right)\right\}_{i=1}^{N_D},
\label{eq:teacher-samples}
\end{equation}
where $N_D=|\mathcal{D}_T|$ is the total number of regression samples,
$\boldsymbol{x}_i^T$ is a sample location, and
$\boldsymbol{y}_i^T\in\mathbb{R}^{m}$ is the corresponding regression
target. Symbolic searches associated with the same teacher share this
fixed dataset; data and candidate pools remain separate across teachers.

\subsubsection{Symbolic search and candidate pooling}
\label{sec:search-pooling}

Using the regression data, PySR~\cite{cranmer2023interpretable} searches
for candidate expressions for each solution component. For a fixed
component $k$, candidate fit is measured by the regression loss
\begin{equation}
\mathcal{L}_T(\mathcal{T},\boldsymbol{a})
=\frac{1}{N_D}\sum_{i=1}^{N_D}
\left[s_{\mathcal{T}}(\boldsymbol{x}_i^T;\boldsymbol{a})
-y_{i,k}^T\right]^2,
\label{eq:teacher-search-loss}
\end{equation}
where $y_{i,k}^T$ is the $k$th component of $\boldsymbol{y}_i^T$.
Candidate quality is assessed through this loss and the complexity
$C(\mathcal{T})$ reported by PySR.

For each teacher and solution component, usable candidates from the
loss--complexity frontiers of $K$ searches on the fixed dataset are
combined into $\mathcal{F}_T$. The searches share a common budget.
Repeated parsed expressions are represented by the record with the
lowest search loss, and no additional Pareto filter is applied to the
pooled candidates.

For a single solution component, Algorithm~\ref{alg:candidate-selection}
selects at most $M$ representatives from $\mathcal{F}_T$. It uses logarithmic
loss to accommodate differences across orders of magnitude and normalizes
both loss and complexity before measuring distances. When at least two
candidates can be retained, selection starts with the minimum-complexity
and minimum-loss candidates, which may coincide. Further candidates are
added by maximizing their distance to the nearest selected candidate.
This procedure spreads the retained candidates across the observed
loss--complexity space. Their structures and fitted coefficients form
$\mathcal{P}_T$ for physical refinement. The coupled-field construction
is described in Section~\ref{sec:coupled-recovery}.

\begin{algorithm}[htbp]
\caption{Representative candidate selection}
\label{alg:candidate-selection}
\small
\begin{algorithmic}[1]
\Require Nonempty pool $\mathcal{F}_T=\{(\mathcal{T}_i,\boldsymbol{a}_i^{\mathrm{search}})\}_{i=1}^{n}$
\Require Search losses $\phi_i=\mathcal{L}_T(\mathcal{T}_i,\boldsymbol{a}_i^{\mathrm{search}})$ and complexities $C_i=C(\mathcal{T}_i)$
\Require Candidate cap $M\in\mathbb{N}_{>0}$, numerical floor $\varepsilon_{\mathrm{floor}}>0$
\Ensure Representative candidate pool $\mathcal{P}_T$
\State Sort candidates by increasing complexity, then loss; preserve the original order for remaining ties.
\State Resolve all subsequent ties by this candidate order.
\State $M_T\gets\min\{M,n\}$
\State $\psi_i\gets\log(\max\{\phi_i,\varepsilon_{\mathrm{floor}}\})$, $i=1,\ldots,n$
\State $C_{\min}\gets\min_i C_i$, $C_{\max}\gets\max_i C_i$
\State $\psi_{\min}\gets\min_i\psi_i$, $\psi_{\max}\gets\max_i\psi_i$
\State $\displaystyle\boldsymbol{z}_i\gets\left(\frac{C_i-C_{\min}}{C_{\max}-C_{\min}},\frac{\psi_i-\psi_{\min}}{\psi_{\max}-\psi_{\min}}\right)$, $i=1,\ldots,n$
\Statex \hspace{\algorithmicindent}Set a coordinate to zero if its denominator is zero.
\State $i_C\gets\arg\min_i C_i$, $i_\phi\gets\arg\min_i \phi_i$
\If{$M_T=1$}
    \State $S\gets\{i_\phi\}$
\Else
    \State $S\gets\{i_C,i_\phi\}$
\EndIf
\While{$|S|<M_T$}
    \State $\displaystyle i^\star\gets\arg\max_{i\notin S}\min_{j\in S}\|\boldsymbol{z}_i-\boldsymbol{z}_j\|_2$
    \State $S\gets S\cup\{i^\star\}$
\EndWhile
\State $\mathcal{P}_T\gets\{(\mathcal{T}_i,\boldsymbol{a}_i^{\mathrm{search}}):i\in S\}$
\State \Return $\mathcal{P}_T$
\end{algorithmic}
\end{algorithm}

\subsection{Physics-based coefficient refinement and model selection}
\label{sec:coefficient-recovery}

The retained candidates approximate the regression targets, but their fitted
coefficients need not satisfy the original equations and constraints.
We therefore use these estimates to initialize physics-only refinement
with each structure and its parameterization fixed. The refined expressions
are subsequently assessed by a physical score, and complexity is used to
select among candidates with near-optimal scores.

\subsubsection{Fixed-structure coefficient refinement}
\label{sec:fixed-structure-recovery}

For a candidate with structure $\mathcal{T}$, let
$s_{\mathcal{T}}(\boldsymbol{x};\boldsymbol{a})$ denote its parameterized
expression, with free coefficients
$\boldsymbol{a}\in U_{\mathcal{T}}\subseteq\mathbb{R}^{p_{\mathcal{T}}}$.
The admissible set $U_{\mathcal{T}}$ comprises parameter values for which
the expression and its sampled equation and constraint residuals are well
defined and finite.

For scalar candidates, the parameterization retains numerical constants
in power exponents and constants numerically equal to $0$ or $\pm1$ as
fixed parts of the expression. The remaining numerical constants define the adjustable coefficients.
Constants with repeated values, identified within a prescribed tolerance,
share a parameter throughout refinement. The parameter cap and the rule
for holding excess coefficient groups fixed are specified in
Section~\ref{sec:experimental-setup}. The resulting search-derived values
form the initial vector $\boldsymbol{a}_{\mathcal{T}}^{(0)}$.
Refinement preserves all fixed constants and parameter-sharing relations.

The equation and constraint residuals are
\begin{align}
r_{\mathcal{T}}(\boldsymbol{x};\boldsymbol{a})
&=\mathcal{N}[s_{\mathcal{T}}](\boldsymbol{x};\boldsymbol{a})
-f(\boldsymbol{x}),
\label{eq:recovery-pde-residual}\\
c_{\ell,\mathcal{T}}(\boldsymbol{x};\boldsymbol{a})
&=\mathcal{B}_{\ell}[s_{\mathcal{T}}](\boldsymbol{x};\boldsymbol{a})
-g_{\ell}(\boldsymbol{x}),
\qquad \ell=1,\ldots,L.
\label{eq:recovery-constraint-residual}
\end{align}
SymPy~\cite{meurer2017sympy} computes the required coordinate derivatives
symbolically, and the residual expressions are converted into numerical
functions. At fixed interior and constraint points,
$\boldsymbol{R}_{0,\mathcal{T}}(\boldsymbol{a})$ and
$\boldsymbol{R}_{\ell,\mathcal{T}}(\boldsymbol{a})$ stack the PDE residuals
and those of constraint group $\ell$, respectively.

\begin{samepage}
The physical refinement objective is
\begin{equation}
\mathcal{J}_{\mathrm{phys}}(\mathcal{T},\boldsymbol{a})
=\left\|\boldsymbol{R}_{0,\mathcal{T}}(\boldsymbol{a})\right\|_2^2
+w_{\mathrm{c}}\sum_{\ell=1}^{L}
\left\|\boldsymbol{R}_{\ell,\mathcal{T}}(\boldsymbol{a})\right\|_2^2,
\label{eq:coefficient-recovery-objective}
\end{equation}
where $w_{\mathrm{c}}>0$ weights the constraints relative to the
governing equation. We set $w_{\mathrm{c}}=100$ in all experiments.
\par\end{samepage}

The coefficients are refined by solving the nonlinear least-squares problem
\begin{equation}
\min_{\boldsymbol{a}\in U_{\mathcal{T}}}
\mathcal{J}_{\mathrm{phys}}(\mathcal{T},\boldsymbol{a}).
\label{eq:coefficient-recovery-problem}
\end{equation}
The sampling points and weights remain fixed during optimization.
Local least-squares solves start from
$\boldsymbol{a}_{\mathcal{T}}^{(0)}$ and its random perturbations.
Among the optimization outputs with finite objective values, we retain
the coefficient vector with the smallest physical objective as
$\widehat{\boldsymbol{a}}_{\mathcal{T}}$, and its convergence
status is recorded. The refined expression is
$s_{\mathcal{T}}(\boldsymbol{x};\widehat{\boldsymbol{a}}_{\mathcal{T}})$.
Expressions with no free coefficients are evaluated without optimization.

\subsubsection{Physics-based model selection}
\label{sec:selection-verification}

Model selection compares the physical quality of the refined candidates
and favors simpler expressions when their scores are close.
Let $\mathcal{E}_T$ denote the admissible refined expressions with finite
physical scores obtained from $\mathcal{P}_T$. All candidates are evaluated
at the common sampling points used for refinement. For $s\in\mathcal{E}_T$,
let $\boldsymbol{R}_0(s)$ and $\boldsymbol{R}_{\ell}(s)$ contain the equation
and constraint residuals, with $n_0$ and $n_\ell$ scalar entries,
respectively. To remove the direct dependence of residual magnitude on
entry counts, we use the RMS-based score

\begin{equation}
\begin{aligned}
e_r(s)&=\frac{\|\boldsymbol{R}_0(s)\|_2}{\sqrt{n_0}},\\[6pt]
e_c(s)&=\sqrt{\frac{\sum_{\ell=1}^{L}\|\boldsymbol{R}_\ell(s)\|_2^2}
{\sum_{\ell=1}^{L}n_\ell}},\\[8pt]
Q_{\mathrm{phys}}(s)&=e_r(s)+\sqrt{w_{\mathrm{c}}}\,e_c(s).
\end{aligned}
\label{eq:model-selection-score}
\end{equation}
The score combines equation and constraint RMS values, using
$\sqrt{w_{\mathrm{c}}}$ as the relative weight on the constraint term.
Unlike the squared-residual objective in
\eqref{eq:coefficient-recovery-objective}, which determines coefficients
within a fixed structure, this score is used to compare refined candidates.

To favor simpler expressions when physical scores are close, we restrict
selection to a physically near-optimal subset.
Let
\mbox{$Q_{\min}=\min_{s\in\mathcal{E}_T}Q_{\mathrm{phys}}(s)$}. The physically
near-optimal subset is
\begin{equation}
\mathcal{A}_T=\left\{s\in\mathcal{E}_T:
Q_{\mathrm{phys}}(s)\leq
(1+\delta_{\mathrm{rel}})Q_{\min}+\delta_{\mathrm{abs}}\right\}.
\label{eq:physical-near-optimal-subset}
\end{equation}
The model-selection tolerances
$\delta_{\mathrm{rel}},\delta_{\mathrm{abs}}\geq0$ permit a score increase
of at most $\delta_{\mathrm{rel}}Q_{\min}+\delta_{\mathrm{abs}}$ above the
minimum. The relative term limits proportional degradation in physical
score, while the absolute term sets the comparison scale near zero.
We use $\delta_{\mathrm{rel}}=0.05$ and
$\delta_{\mathrm{abs}}=10^{-10}$ in all experiments.
The final expression $s_\star$ is the candidate in $\mathcal{A}_T$ with
the smallest expression-tree node count $C_{\mathrm{sel}}(s)$.

\subsection{Coupled multi-field symbolic recovery}
\label{sec:coupled-recovery}

For coupled PDE systems, the component fields are linked through the
governing equations and constraints. DeSyR accounts for these relations
by combining fieldwise symbolic search with joint construction, physical
screening and coefficient refinement of system candidates.

For a teacher with $m$ outputs, fieldwise symbolic regression on
$\mathcal{D}_T$ yields candidate sets $\mathcal{F}_T^{(k)}$,
$k=1,\ldots,m$. Recovery seeks compatible combinations of these
component expressions, represented by the ordered tuple
\begin{equation}
\boldsymbol{s}_{\boldsymbol{\mathcal{T}}}
(\boldsymbol{x};\boldsymbol{a})
=\bigl(s_{\mathcal{T}_1}(\boldsymbol{x};\boldsymbol{a}),\ldots,
s_{\mathcal{T}_m}(\boldsymbol{x};\boldsymbol{a})\bigr)^{\mathsf T},
\label{eq:coupled-symbolic-candidate}
\end{equation}
where $\boldsymbol{\mathcal{T}}=(\mathcal{T}_1,\ldots,\mathcal{T}_m)$
specifies the component structures and $\boldsymbol{a}$ collects their
free coefficients.

Fieldwise candidate sets may be enriched by recombining subexpressions
already identified by symbolic search. The added expressions are
initialized by fitting their coefficients to teacher samples.

Each system candidate is first assessed using its provisional coefficients.
Equation residuals are concatenated across equations, and constraint
residuals across constraint groups, to evaluate the physical score in
\eqref{eq:model-selection-score}. The lowest-scoring combinations, up to
the prescribed screening cap, form $\mathcal{P}_T$ for joint refinement;
the problem-dependent caps are reported in
Section~\ref{sec:experimental-setup}. Because this screening precedes
refinement, a suitable structure may be excluded when its provisional
coefficients produce large residuals.

With the component structures and their parameterization fixed, the joint
coefficient vector is
refined using \eqref{eq:coefficient-recovery-objective}, extended to
include all equation and constraint residuals. Final selection follows
Section~\ref{sec:selection-verification}, using the total expression-tree complexity
\mbox{$C_{\mathrm{sel}}(\boldsymbol{s})=
\sum_{k=1}^{m}C_{\mathrm{sel}}(s_k)$}.
Algorithm~\ref{alg:desyr-framework} summarizes the single-field and coupled
procedures, applied independently to each teacher.

\begin{algorithm}[htbp]
\caption{The DeSyR framework}
\label{alg:desyr-framework}
\small
\begin{algorithmic}[1]
\Require Governing equations and constraints; operator library; sampling, search and refinement settings
\Ensure Final symbolic expression $\boldsymbol{s}_\star$
\State Train a PINN teacher and construct $\mathcal{D}_T$ (Section~\ref{sec:teacher-sampling}).
\State Perform $K$ symbolic searches for each output field using the fixed $\mathcal{D}_T$.
\State Merge and deduplicate search results separately for each field.
\If{only one solution component is sought}
    \State Construct $\mathcal{P}_T$ using Algorithm~\ref{alg:candidate-selection}.
\Else
    \State Enrich fieldwise candidate sets by optional subexpression recombination.
    \State Form system candidates and select $\mathcal{P}_T$ by joint physical consistency.
\EndIf
\For{each candidate in $\mathcal{P}_T$}
    \State Fix its structure and parameterization.
    \State Refine its free coefficients by physical least squares.
\EndFor
\State Select the minimum-complexity candidate in the physically near-optimal subset as $\boldsymbol{s}_\star$ (Section~\ref{sec:selection-verification}).
\State \Return $\boldsymbol{s}_\star$.
\end{algorithmic}
\end{algorithm}

\section{Theoretical analysis}
\label{sec:theory}

DeSyR uses teacher predictions to discover candidate structures and
physical residuals to refine their coefficients. For a fixed candidate,
we examine how this separation affects coefficient estimation. We first
characterize the propagation of teacher errors, then establish conditions
for identifiable and stable recovery under physical constraints.

\subsection{Mathematical setting and assumptions}
\label{sec:theory-setting}

To examine the effect of the coefficient-estimation objective, we fix a
candidate structure and its parameterization as defined in
Section~\ref{sec:fixed-structure-recovery}. Let
$\boldsymbol{s}(\boldsymbol{x};\boldsymbol{a})\in\mathbb{R}^{m}$ denote
the resulting expression, with $p\geq1$ free coefficients
$\boldsymbol{a}\in U\subseteq\mathbb{R}^{p}$. For brevity, we suppress
the candidate indices and write $U=U_{\mathcal{T}}$ and $p=p_{\mathcal{T}}$.
The set $U$ contains admissible parameter values for which the sampled
outputs and physical residuals are well defined.

Assume that the target solution $\boldsymbol{u}_\star$ of
\eqref{eq:governing-problem}--\eqref{eq:prescribed-constraints} admits
an exact representation within this parameterization: there exists
$\boldsymbol{a}_\star\in U$ such that
\begin{equation}
\boldsymbol{s}(\boldsymbol{x};\boldsymbol{a}_\star)
=\boldsymbol{u}_\star(\boldsymbol{x}),
\qquad \boldsymbol{x}\in\Omega,
\label{eq:theory-representability}
\end{equation}
with the prescribed constraints satisfied. Exact representability is
assumed for the complete fixed parameterization, including its frozen
constants and parameter-sharing relations. We select one representing
vector $\boldsymbol{a}_\star$ as the reference for the analysis.

We use the fixed regression dataset $\mathcal{D}_T$ from
Section~\ref{sec:teacher-sampling} and omit the superscript $T$ on its
sample coordinates and labels. Thus, $\boldsymbol{x}_i=\boldsymbol{x}_i^T$
and $\boldsymbol{y}_i=\boldsymbol{y}_i^T$, $i=1,\ldots,N_D$.
Define the normalized output and label vectors by

\begin{equation}
\begin{aligned}
\boldsymbol{S}(\boldsymbol{a})
&=\frac{1}{\sqrt{N_D}}
\begin{bmatrix}
\boldsymbol{s}(\boldsymbol{x}_1;\boldsymbol{a})\\
\vdots\\
\boldsymbol{s}(\boldsymbol{x}_{N_D};\boldsymbol{a})
\end{bmatrix},\\[6pt]
\boldsymbol{y}_T
&=\frac{1}{\sqrt{N_D}}
\begin{bmatrix}
\boldsymbol{y}_1\\
\vdots\\
\boldsymbol{y}_{N_D}
\end{bmatrix}.
\end{aligned}
\label{eq:theory-data-vectors}
\end{equation}
Both vectors lie in $\mathbb{R}^{n_D}$, where $n_D=mN_D$ is the total
number of scalar data entries. The factor $N_D^{-1/2}$ makes the squared
fitting error equal to the sum of the fieldwise mean-squared losses in
\eqref{eq:teacher-search-loss}. We express the normalized regression-data
error as
\begin{equation}
\boldsymbol{y}_T
=\boldsymbol{S}(\boldsymbol{a}_\star)+\boldsymbol{e}.
\label{eq:theory-teacher-error}
\end{equation}
The components of $\boldsymbol{e}$ are teacher prediction errors scaled
by $N_D^{-1/2}$, with zero entries for exact constraint labels.

At the fixed refinement points, let $\boldsymbol{R}_0$ collect the
governing-equation residuals and $\boldsymbol{R}_\ell$ collect the residuals of
constraint group $\ell$, $\ell=1,\ldots,L$, including any prescribed
gauge condition. For coupled problems, these blocks collect all relevant
components. The weighted residual vector is
\begin{equation}
\boldsymbol{R}(\boldsymbol{a})
=\begin{bmatrix}
\boldsymbol{R}_0(\boldsymbol{a})\\
\sqrt{w_{\mathrm c}}\,\boldsymbol{R}_1(\boldsymbol{a})\\
\vdots\\
\sqrt{w_{\mathrm c}}\,\boldsymbol{R}_L(\boldsymbol{a})
\end{bmatrix}
\in\mathbb{R}^{n_R},
\label{eq:theory-weighted-residual}
\end{equation}
where $n_R$ is the total number of scalar residual entries. With exact
physical data and exact residual evaluation, the representability
assumption \eqref{eq:theory-representability} implies
$\boldsymbol{R}(\boldsymbol{a}_\star)=\boldsymbol{0}$.

These sampled maps define the teacher-fitting and physical objectives
\begin{equation}
\mathcal{J}_T(\boldsymbol{a})
=\|\boldsymbol{S}(\boldsymbol{a})-\boldsymbol{y}_T\|_2^2,
\qquad
\mathcal{J}_P(\boldsymbol{a})
=\|\boldsymbol{R}(\boldsymbol{a})\|_2^2.
\label{eq:theory-comparison-objectives}
\end{equation}
Here, $\mathcal{J}_P$ is the fixed-candidate refinement objective in
\eqref{eq:coefficient-recovery-objective}. To compare how teacher error
affects coefficient estimates under different fitting objectives, we
also introduce the mixed objective
\begin{equation}
\mathcal{J}_\eta(\boldsymbol{a})
=\mathcal{J}_P(\boldsymbol{a})+\eta\mathcal{J}_T(\boldsymbol{a}),
\qquad \eta\geq0.
\label{eq:theory-mixed-objective}
\end{equation}
The weight $\eta$ controls the contribution of teacher fitting in this
theoretical comparison. DeSyR uses $\mathcal{J}_P$ for coefficient
refinement.

\subsection{Teacher-error propagation in coefficient estimation}
\label{sec:teacher-error-propagation}

We first analyze teacher-error propagation when both the sampled outputs and
physical residuals depend affinely on the free coefficients. With
$\boldsymbol{h}=\boldsymbol{a}-\boldsymbol{a}_\star$, write
\begin{equation}
\boldsymbol{S}(\boldsymbol{a})
=\boldsymbol{S}(\boldsymbol{a}_\star)+B\boldsymbol{h},
\qquad
\boldsymbol{R}(\boldsymbol{a})=A\boldsymbol{h},
\label{eq:affine-coefficient-maps}
\end{equation}
where $B\in\mathbb{R}^{n_D\times p}$ and
$A\in\mathbb{R}^{n_R\times p}$ include the sampling and weighting
defined in Section~\ref{sec:theory-setting}. The second identity uses
$\boldsymbol{R}(\boldsymbol{a}_\star)=\boldsymbol{0}$. The affine setting
concerns dependence on the free coefficients. The results below apply to
unconstrained least-squares minimizers in $\mathbb{R}^{p}$ that lie in $U$.
Nonlinear coefficient dependence, including adjustable frequencies and
decay rates, is treated subsequently through a local perturbation analysis.

\begin{proposition}[Teacher-induced coefficient error]
\label{prop:teacher-coefficient-error}
Suppose \eqref{eq:affine-coefficient-maps} holds and $B$ has full column
rank. The unique minimizer $\widehat{\boldsymbol{a}}_T$ of
$\mathcal{J}_T$ satisfies
\begin{equation}
\widehat{\boldsymbol{a}}_T-\boldsymbol{a}_\star
=(B^{\mathsf T}B)^{-1}B^{\mathsf T}\boldsymbol{e},
\label{eq:teacher-coefficient-error}
\end{equation}
and hence
\begin{equation}
\|\widehat{\boldsymbol{a}}_T-\boldsymbol{a}_\star\|_2
\leq\frac{\|\boldsymbol{e}\|_2}{\sigma_{\min}(B)},
\label{eq:teacher-coefficient-error-bound}
\end{equation}
where $\sigma_{\min}(B)>0$ is the smallest singular value of $B$.
\end{proposition}

\begin{proof}
Equations~\eqref{eq:theory-teacher-error} and
\eqref{eq:affine-coefficient-maps} give
$\mathcal{J}_T(\boldsymbol{a})=\|B\boldsymbol{h}-\boldsymbol{e}\|_2^2$.
Full column rank makes $B^{\mathsf T}B$ positive definite, so the unique
minimizer solves
$B^{\mathsf T}B\boldsymbol{h}=B^{\mathsf T}\boldsymbol{e}$.
This yields \eqref{eq:teacher-coefficient-error}. The bound follows from
$\|(B^{\mathsf T}B)^{-1}B^{\mathsf T}\|_2=1/\sigma_{\min}(B)$,
using the induced Euclidean matrix norm.
\end{proof}

Proposition~\ref{prop:teacher-coefficient-error} identifies the component
of teacher error that enters the coefficient estimate: its projection onto
the column space of $B$. Errors orthogonal to this space leave the reference
coefficients unchanged, whereas components within it are absorbed into the
fitted coefficients. The smallest singular value $\sigma_{\min}(B)$
quantifies the worst-direction sensitivity of this estimation.

We next examine coefficient estimation when physical refinement retains
a teacher-fitting term, using the comparison objective $\mathcal{J}_\eta$ in
\eqref{eq:theory-mixed-objective}.

\begin{proposition}[Coefficient error under a mixed objective]
\label{prop:mixed-coefficient-error}
Suppose \eqref{eq:affine-coefficient-maps} holds and
$H_\eta=A^{\mathsf T}A+\eta B^{\mathsf T}B$ is positive definite for
some $\eta\geq0$. The unique minimizer
$\widehat{\boldsymbol{a}}_\eta$ of $\mathcal{J}_\eta$ satisfies
\begin{equation}
\widehat{\boldsymbol{a}}_\eta-\boldsymbol{a}_\star
=\eta H_\eta^{-1}B^{\mathsf T}\boldsymbol{e},
\label{eq:mixed-coefficient-error}
\end{equation}
with the estimate
\begin{equation}
\|\widehat{\boldsymbol{a}}_\eta-\boldsymbol{a}_\star\|_2
\leq
\frac{\eta\|B\|_2}{\lambda_{\min}(H_\eta)}
\|\boldsymbol{e}\|_2,
\label{eq:mixed-coefficient-error-bound}
\end{equation}
where $\lambda_{\min}(H_\eta)>0$ is the smallest eigenvalue of
$H_\eta$.
\end{proposition}

\begin{proof}
In terms of $\boldsymbol{h}$, the mixed objective is
$\|A\boldsymbol{h}\|_2^2+
\eta\|B\boldsymbol{h}-\boldsymbol{e}\|_2^2$.
Its normal equations are
$H_\eta\boldsymbol{h}=\eta B^{\mathsf T}\boldsymbol{e}$.
Positive definiteness gives the unique solution
\eqref{eq:mixed-coefficient-error}. Applying the induced matrix norm and
$\|H_\eta^{-1}\|_2=1/\lambda_{\min}(H_\eta)$ yields
\eqref{eq:mixed-coefficient-error-bound}.
\end{proof}

Proposition~\ref{prop:mixed-coefficient-error} shows that, for $\eta>0$,
retaining the teacher-fitting term shifts the coefficient
estimate away from $\boldsymbol{a}_\star$ unless the teacher error is
orthogonal to the sampled coefficient sensitivities
($B^{\mathsf T}\boldsymbol{e}=\boldsymbol{0}$). If $A$ has full column
rank, then $\lambda_{\min}(H_\eta)\geq\sigma_{\min}(A)^2$, so
\eqref{eq:mixed-coefficient-error-bound} also gives
\begin{equation}
\|\widehat{\boldsymbol{a}}_\eta-\boldsymbol{a}_\star\|_2
\leq\frac{\eta\|B\|_2}{\sigma_{\min}(A)^2}
\|\boldsymbol{e}\|_2.
\label{eq:mixed-error-physical-rank-bound}
\end{equation}
For fixed $A$, $B$ and $\boldsymbol{e}$, full column rank of $A$ provides
a uniform positive lower bound on $\lambda_{\min}(H_\eta)$. The resulting
coefficient-error bound therefore vanishes as $\eta\to0$, linking this
limit to identifiability through the physical residual map.

The affine identities have a local counterpart for nonlinear coefficient
maps. In the following perturbation analysis, the candidate,
parameter-sharing rules and sampling points are held fixed while the
label error $\boldsymbol{e}$ in \eqref{eq:theory-teacher-error} varies.

\begin{proposition}[Local nonlinear teacher-error propagation]
\label{prop:nonlinear-teacher-error}
Suppose $\boldsymbol{S}$ and $\boldsymbol{R}$ are twice continuously
differentiable on an open neighborhood $V\subseteq U$ of
$\boldsymbol{a}_\star$, and
$\boldsymbol{R}(\boldsymbol{a}_\star)=\boldsymbol{0}$. Define
\begin{equation}
\begin{aligned}
B_\star&=D\boldsymbol{S}(\boldsymbol{a}_\star),
\qquad A_\star=D\boldsymbol{R}(\boldsymbol{a}_\star),\\
H_{\eta,\star}&=A_\star^{\mathsf T}A_\star
+\eta B_\star^{\mathsf T}B_\star.
\end{aligned}
\label{eq:nonlinear-coefficient-jacobians}
\end{equation}
\begin{enumerate}
\renewcommand{\labelenumi}{(\roman{enumi})}
\item If $B_\star$ has full column rank, then, for all sufficiently small
$\boldsymbol{e}$, $\mathcal{J}_T$ has a unique stationary point
$\widehat{\boldsymbol{a}}_T(\boldsymbol{e})$ in a fixed neighborhood
of $\boldsymbol{a}_\star$, with
\begin{equation}
\widehat{\boldsymbol{a}}_T(\boldsymbol{e})-\boldsymbol{a}_\star
=(B_\star^{\mathsf T}B_\star)^{-1}B_\star^{\mathsf T}\boldsymbol{e}
+o(\|\boldsymbol{e}\|_2).
\label{eq:nonlinear-teacher-coefficient-error}
\end{equation}
\item For a fixed $\eta>0$ such that $H_{\eta,\star}$ is positive
definite, $\mathcal{J}_\eta$ has a unique stationary point
$\widehat{\boldsymbol{a}}_\eta(\boldsymbol{e})$ in a fixed neighborhood
of $\boldsymbol{a}_\star$ for all sufficiently small $\boldsymbol{e}$,
with
\begin{equation}
\widehat{\boldsymbol{a}}_\eta(\boldsymbol{e})-\boldsymbol{a}_\star
=\eta H_{\eta,\star}^{-1}B_\star^{\mathsf T}\boldsymbol{e}
+o(\|\boldsymbol{e}\|_2).
\label{eq:nonlinear-mixed-coefficient-error}
\end{equation}
\end{enumerate}
Each branch consists of strict local minimizers, is continuously
differentiable in $\boldsymbol{e}$, and equals $\boldsymbol{a}_\star$
at $\boldsymbol{e}=\boldsymbol{0}$. Both expansions hold as
$\boldsymbol{e}\to\boldsymbol{0}$. For the mixed branch, reducing
the error neighborhood if necessary also gives
\begin{equation}
\|\widehat{\boldsymbol{a}}_\eta(\boldsymbol{e})
-\boldsymbol{a}_\star\|_2
\leq\frac{2\eta\|B_\star^{\mathsf T}\boldsymbol{e}\|_2}
{\lambda_{\min}(H_{\eta,\star})}.
\label{eq:nonlinear-mixed-coefficient-bound}
\end{equation}
On this mixed branch,
$\widehat{\boldsymbol{a}}_\eta(\boldsymbol{e})=\boldsymbol{a}_\star$
if and only if $B_\star^{\mathsf T}\boldsymbol{e}=\boldsymbol{0}$.
\end{proposition}

\begin{proof}
Define $\boldsymbol{G}_T=\tfrac12\nabla_{\boldsymbol{a}}\mathcal{J}_T$
and $\boldsymbol{G}_\eta=\tfrac12\nabla_{\boldsymbol{a}}\mathcal{J}_\eta$.
Then
\begin{equation}
\begin{aligned}
\boldsymbol{G}_T(\boldsymbol{a},\boldsymbol{e})
&=D\boldsymbol{S}(\boldsymbol{a})^{\mathsf T}
\bigl[\boldsymbol{S}(\boldsymbol{a})
-\boldsymbol{S}(\boldsymbol{a}_\star)-\boldsymbol{e}\bigr],\\
\boldsymbol{G}_\eta(\boldsymbol{a},\boldsymbol{e})
&=D\boldsymbol{R}(\boldsymbol{a})^{\mathsf T}
\boldsymbol{R}(\boldsymbol{a})
+\eta\boldsymbol{G}_T(\boldsymbol{a},\boldsymbol{e}).
\end{aligned}
\label{eq:nonlinear-fitting-stationarity}
\end{equation}
At $(\boldsymbol{a}_\star,\boldsymbol{0})$, both gradients vanish,
and their derivatives are
$D_{\boldsymbol{a}}\boldsymbol{G}_T=B_\star^{\mathsf T}B_\star$,
$D_{\boldsymbol{e}}\boldsymbol{G}_T=-B_\star^{\mathsf T}$,
$D_{\boldsymbol{a}}\boldsymbol{G}_\eta=H_{\eta,\star}$ and
$D_{\boldsymbol{e}}\boldsymbol{G}_\eta=-\eta B_\star^{\mathsf T}$.
The implicit function theorem gives the respective locally unique
stationary branches and their derivatives, yielding
\eqref{eq:nonlinear-teacher-coefficient-error} and
\eqref{eq:nonlinear-mixed-coefficient-error}. Continuity of the
Hessians preserves positive definiteness for small errors, so these
stationary points are strict local minimizers.

For the mixed branch, shrink to a convex parameter ball and an error
neighborhood on which
$D_{\boldsymbol{a}}\boldsymbol{G}_\eta
\succeq\lambda_{\min}(H_{\eta,\star})I/2$.
Write $\boldsymbol{h}=\widehat{\boldsymbol{a}}_\eta(\boldsymbol{e})
-\boldsymbol{a}_\star$. Stationarity and
$\boldsymbol{G}_\eta(\boldsymbol{a}_\star,\boldsymbol{e})
=-\eta B_\star^{\mathsf T}\boldsymbol{e}$ give
\begin{equation}
\eta B_\star^{\mathsf T}\boldsymbol{e}
=\int_0^1
D_{\boldsymbol{a}}\boldsymbol{G}_\eta
(\boldsymbol{a}_\star+t\boldsymbol{h},\boldsymbol{e})
\boldsymbol{h}\,\mathrm{d}t.
\label{eq:nonlinear-mixed-gradient-integral}
\end{equation}
Taking the inner product with $\boldsymbol{h}$ and applying
Cauchy--Schwarz yields \eqref{eq:nonlinear-mixed-coefficient-bound}.
Finally, $\boldsymbol{a}_\star$ is stationary exactly when
$B_\star^{\mathsf T}\boldsymbol{e}=\boldsymbol{0}$; local uniqueness
proves the stated equivalence.
\end{proof}

These results extend the affine identities to first-order perturbation
laws near the reference coefficients and preserve the exact orthogonality
criterion on the local mixed branch. The perturbation statements are
formulated for each fixed $\eta>0$, with the corresponding parameter and
error neighborhoods.

These results clarify the separation used in DeSyR: teacher predictions
guide candidate discovery and initialization, while the physics-only
refinement objective removes the direct teacher-error contribution
associated with fitting those predictions. We next examine how the physical
residual map determines coefficient identifiability and controls recovery
errors.

\subsection{Identifiability and stability under physical constraints}
\label{sec:recovery-identifiability-stability}

For the physical objective $\mathcal{J}_P$, coefficient identifiability
and sensitivity to residual errors are governed by the weighted residual
map.

\begin{proposition}[Affine recovery and stability]
\label{prop:affine-physical-recovery}
Suppose $\boldsymbol{R}(\boldsymbol{a})
=A(\boldsymbol{a}-\boldsymbol{a}_\star)$ and $A$ has full column rank.
Then $\boldsymbol{a}_\star$ is the unique minimizer of
$\mathcal{J}_P$ over $U$. For any $\widehat{\boldsymbol{a}}\in U$
satisfying $\|\boldsymbol{R}(\widehat{\boldsymbol{a}})\|_2
\leq\varepsilon$, where $\varepsilon\geq0$,
\begin{equation}
\|\widehat{\boldsymbol{a}}-\boldsymbol{a}_\star\|_2
\leq\frac{\varepsilon}{\sigma_{\min}(A)}.
\label{eq:affine-physical-stability}
\end{equation}
\end{proposition}

\begin{proof}
Full column rank gives
\begin{equation}
\|\boldsymbol{R}(\boldsymbol{a})\|_2
=\|A(\boldsymbol{a}-\boldsymbol{a}_\star)\|_2
\geq\sigma_{\min}(A)
\|\boldsymbol{a}-\boldsymbol{a}_\star\|_2.
\label{eq:affine-residual-lower-bound}
\end{equation}
Since $\mathcal{J}_P(\boldsymbol{a}_\star)=0$, this inequality proves
uniqueness and yields \eqref{eq:affine-physical-stability}.
\end{proof}

The bound in Proposition~\ref{prop:affine-physical-recovery} involves
the physical residual and its sensitivity to the coefficients, but no
teacher-error term. In contrast to
\eqref{eq:teacher-coefficient-error-bound}, it permits exact recovery
from exact physical constraints even when the teacher predictions are
inexact. A small $\sigma_{\min}(A)$ permits larger coefficient errors
along weakly constrained directions for a given residual level.
Full column rank is a sufficient identifiability
condition for this fixed representation; it is not a consequence of
using a large number of residual points alone.

For nonlinear parameterizations, define
$J_R(\boldsymbol{a})=D\boldsymbol{R}(\boldsymbol{a})$, the Jacobian of
the weighted residual vector with respect to the free coefficients.
In particular, $J_R(\boldsymbol{a}_\star)=A_\star$.

\begin{proposition}[Local nonlinear recovery and stability]
\label{prop:nonlinear-physical-recovery}
Suppose $\boldsymbol{R}(\boldsymbol{a}_\star)=\boldsymbol{0}$ and
$\boldsymbol{R}$ is continuously differentiable on an open neighborhood
of the closed ball
$\overline{\mathbb{B}}_{\rho_0}(\boldsymbol{a}_\star)\subset U$, with
$\rho_0>0$. Assume $J_R(\boldsymbol{a}_\star)$ has full column rank and
write $\sigma_\star=\sigma_{\min}(J_R(\boldsymbol{a}_\star))>0$.
If, on this ball,
\begin{equation}
\|J_R(\boldsymbol{a})-J_R(\boldsymbol{b})\|_2
\leq \gamma_R\|\boldsymbol{a}-\boldsymbol{b}\|_2
\label{eq:residual-jacobian-lipschitz}
\end{equation}
for some $\gamma_R\geq0$, choose $0<\rho\leq\rho_0$ such that
$\gamma_R\rho\leq\sigma_\star$. Then
\begin{equation}
\|\boldsymbol{R}(\boldsymbol{a})\|_2
\geq\frac{\sigma_\star}{2}
\|\boldsymbol{a}-\boldsymbol{a}_\star\|_2,
\qquad
\boldsymbol{a}\in\overline{\mathbb{B}}_{\rho}(\boldsymbol{a}_\star).
\label{eq:nonlinear-residual-lower-bound}
\end{equation}
Consequently, $\boldsymbol{a}_\star$ is the unique zero-residual
parameter in this ball, and any estimate in the ball with residual norm
at most $\varepsilon$ satisfies
\begin{equation}
\|\widehat{\boldsymbol{a}}-\boldsymbol{a}_\star\|_2
\leq\frac{2\varepsilon}{\sigma_\star}.
\label{eq:nonlinear-physical-stability}
\end{equation}
\end{proposition}

\begin{proof}
For $\boldsymbol{h}=\boldsymbol{a}-\boldsymbol{a}_\star$, integration
along the segment from $\boldsymbol{a}_\star$ to $\boldsymbol{a}$ gives
\begin{equation}
\boldsymbol{R}(\boldsymbol{a})
=J_R(\boldsymbol{a}_\star)\boldsymbol{h}
+\int_0^1
\bigl[J_R(\boldsymbol{a}_\star+t\boldsymbol{h})
-J_R(\boldsymbol{a}_\star)\bigr]\boldsymbol{h}\,\mathrm{d}t.
\label{eq:nonlinear-residual-expansion}
\end{equation}
The integral has norm at most $\gamma_R\|\boldsymbol{h}\|_2^2/2$.
The reverse triangle inequality therefore yields
\begin{equation}
\|\boldsymbol{R}(\boldsymbol{a})\|_2
\geq
\left(\sigma_\star-\frac{\gamma_R}{2}\|\boldsymbol{h}\|_2\right)
\|\boldsymbol{h}\|_2
\geq\frac{\sigma_\star}{2}\|\boldsymbol{h}\|_2.
\end{equation}
The last inequality follows from
$\|\boldsymbol{h}\|_2\leq\rho$ and $\gamma_R\rho\leq\sigma_\star$.
Uniqueness of the zero and \eqref{eq:nonlinear-physical-stability}
follow directly.
\end{proof}

Within the stated neighborhood, the residual lower bound makes coefficient
error directly controllable by the achieved physical residual. This local
relation complements the initialization and multistart refinement described
in Section~\ref{sec:fixed-structure-recovery} by characterizing the accuracy
of estimates obtained in that region.

\begin{corollary}[Error in the represented solution]
\label{cor:physical-solution-error}
Let $X$ be a normed space of solution fields. In addition to the hypotheses
of Proposition~\ref{prop:nonlinear-physical-recovery}, suppose
$\boldsymbol{a}\mapsto\boldsymbol{s}(\cdot;\boldsymbol{a})$ is
continuously differentiable into $X$ near the stated ball, with
$\|D_{\boldsymbol{a}}\boldsymbol{s}(\cdot;\boldsymbol{a})
\|_{\mathcal{L}(\mathbb{R}^{p},X)}\leq \Lambda_s$ throughout the ball.
For an estimate in the ball with physical residual norm at most
$\varepsilon$,
\begin{equation}
\|\boldsymbol{s}(\cdot;\widehat{\boldsymbol{a}})
-\boldsymbol{u}_\star\|_X
\leq\frac{2\Lambda_s}{\sigma_\star}\,\varepsilon.
\label{eq:physical-solution-error-bound}
\end{equation}
\end{corollary}

\begin{proof}
Integrating the parameter derivative along the segment joining
$\boldsymbol{a}_\star$ and $\widehat{\boldsymbol{a}}$ gives
$\|\boldsymbol{s}(\cdot;\widehat{\boldsymbol{a}})
-\boldsymbol{s}(\cdot;\boldsymbol{a}_\star)\|_X
\leq \Lambda_s\|\widehat{\boldsymbol{a}}-\boldsymbol{a}_\star\|_2$.
Use \eqref{eq:theory-representability} and
\eqref{eq:nonlinear-physical-stability} to obtain the result.
\end{proof}

These estimates also accommodate bounded errors in physical data or
residual evaluation. If an evaluated residual
$\widetilde{\boldsymbol{R}}$ satisfies
$\|\widetilde{\boldsymbol{R}}(\widehat{\boldsymbol{a}})
-\boldsymbol{R}(\widehat{\boldsymbol{a}})\|_2\leq\tau$
and $\|\widetilde{\boldsymbol{R}}(\widehat{\boldsymbol{a}})\|_2
\leq\varepsilon$, the triangle inequality gives a true residual norm
at most $\varepsilon+\tau$. Under the same hypotheses, $\varepsilon$
is therefore replaced by $\varepsilon+\tau$ in the coefficient and
solution error bounds. This perturbation is distinct from teacher error
in the regression labels.

For a retained, exactly representable candidate, these results bound
coefficient and solution errors in terms of the achieved physical residual
and the conditioning of the fixed parameterization. Together with the
teacher-fitting analysis, they distinguish two error mechanisms:
propagation of regression-label errors through the fitting objective,
and residual-controlled recovery under physical constraints. This
distinction provides a theoretical basis for using teacher predictions
in candidate discovery and physical relations in coefficient refinement.

\section{Numerical experiments}
\label{sec:experiments}

We evaluate the accuracy and physical consistency of DeSyR's symbolic
solutions on seven PDE problems, progressing from scalar evolution
equations to multidimensional elliptic problems and coupled incompressible
flows. The benchmark comparisons assess solution errors, equation and
constraint residuals, and recovered expression structures. A separate
fixed-candidate comparison examines how the coefficient-estimation
objective affects frequency and solution accuracy, complementing the
analysis in Section~\ref{sec:theory}.

\subsection{Experimental setup}
\label{sec:experimental-setup}

We compare DeSyR with three baselines: PINNs, PINN--SR and PISN.
The PINN provides the neural solution approximation and serves as the
common teacher for PINN--SR and DeSyR. PISN instead trains a symbolic
network directly under physical constraints~\cite{majumdar2022physics}.

PINN--SR and DeSyR use the same symbolic searches and pooled candidates.
PINN--SR fits coefficients to teacher labels and selects expressions by
teacher-fitting error, independently for each field in coupled problems.
DeSyR uses physical coefficient refinement and model selection, including
joint screening of coupled-field candidates as described in
Section~\ref{sec:coupled-recovery}. This comparison assesses the combined
effect of these design choices; the fixed-candidate study in
Section~\ref{sec:ablation-sensitivity} examines the coefficient-estimation
objective separately.

Shared settings and evaluation metrics are described below; network sizes,
sampling counts and training budgets specific to each problem are given
in the corresponding benchmark subsections.

For each problem, five PINN teachers are trained from independent
initializations using Adam followed by L-BFGS. For each teacher,
symbolic searches are repeated with three random seeds, and the resulting
candidates are pooled within that teacher. Each teacher constitutes one
repetition for PINN--SR and DeSyR. PISN uses five
independent network initializations.

For symbolic regression, the operator library includes the unary operators
$\{\sin,\cos,\exp,\tanh\}$ and the binary operators
$\{+,-,\times,\div\}$. Each symbolic search uses at most 100 iterations,
with 1,000 evolutionary cycles per iteration and a 900-second time limit.
For scalar problems, Algorithm~\ref{alg:candidate-selection} retains at
most $M=15$ candidates from the pooled searches. For Kovasznay and
Taylor--Green flow, joint screening retains at most 16 and 32 system
candidates, respectively.

Following Section~\ref{sec:fixed-structure-recovery}, numerical constants
are identified with a tolerance of $10^{-9}$, with at most 16 free
parameters per scalar field. When this limit is exceeded, coefficient
groups with larger initial absolute values are retained for refinement,
and the remaining constants are fixed.

Physical refinement uses SciPy's trust-region reflective least-squares
solver with two-point finite-difference Jacobians and function, step and
gradient tolerances of $10^{-8}$. It follows the multistart procedure in
Section~\ref{sec:fixed-structure-recovery}, with the search-derived
coefficients and three perturbed initializations. The limits are 8,000
solver-counted function evaluations per solve and 120 seconds per candidate.

Boundary and initial points for refinement are sampled separately from
those used in PINN training, with the same location counts. Boundary
counts refer to totals over the spatial boundary, with locations shared
across the constrained fields. Initial locations are similarly shared
by the applicable initial constraints. Pointwise gauge constraints
remain at their prescribed locations.

PISN uses separate symbolic-network branches for the output fields.
Each symbolic layer applies its operators to trainable linear
combinations of the preceding features and passes the resulting features,
input coordinates and a constant channel to the next layer; a trainable
linear combination produces each output. The scalar evolution problems
and Helmholtz use four symbolic layers with ten channels per operator,
whereas three-dimensional Poisson uses three layers with twelve channels
per operator. These networks use $\sin$, $\exp$, addition and
multiplication. The coupled-flow networks use three layers with four
channels per operator and the operators $\sin$, $\cos$, $\exp$, $\tanh$,
addition, subtraction and multiplication.

Symbolic-solution errors and physical residuals are evaluated using
expressions at the output precision. Coefficients in displayed formulas
may be further rounded for readability.

Solution accuracy is assessed by the relative $L^2$ error and maximum
absolute error on a common test set
$\mathcal{D}_{\mathrm{test}}=\{\boldsymbol{x}_i\}_{i=1}^{N_{\mathrm{test}}}$,
where $u_\star$ and $\widehat u$ denote the reference and approximate
solutions for the field being evaluated:
\begin{align}
E_{\mathrm{rel}}
&=\left(
\frac{\sum_{i=1}^{N_{\mathrm{test}}}
|\widehat u(\boldsymbol{x}_i)
-u_\star(\boldsymbol{x}_i)|^2}
{\sum_{i=1}^{N_{\mathrm{test}}}
|u_\star(\boldsymbol{x}_i)|^2}
\right)^{1/2},
\label{eq:test-relative-error}\\
E_\infty
&=\max_{1\leq i\leq N_{\mathrm{test}}}
|\widehat u(\boldsymbol{x}_i)
-u_\star(\boldsymbol{x}_i)|.
\label{eq:test-maximum-error}
\end{align}
The relative error requires a nonzero denominator. For coupled problems,
both metrics are evaluated separately for each solution component.

Physical consistency is assessed using residual root-mean-square (RMS)
values. For residual entries $r_1,\ldots,r_n$, the RMS is
\begin{equation}
E_{\mathrm{RMS}}=\left(\frac{1}{n}\sum_{i=1}^{n}|r_i|^2\right)^{1/2}.
\label{eq:test-residual-rms}
\end{equation}
Each governing equation is evaluated by its residual RMS. Constraint
residuals are concatenated across all constraint groups, and $e_c$ is
the unweighted RMS of the resulting scalar entries. For scalar problems,
$e_r$ denotes the equation-residual RMS, following the definitions in
\eqref{eq:model-selection-score}. For coupled flows, the momentum and
continuity residuals are reported separately using the notation introduced
in Section~\ref{sec:experiment-navier-stokes}.

\subsection{Time-dependent scalar problems}
\label{sec:time-dependent-scalar-problems}

The scalar evolution benchmarks comprise a separable diffusion solution,
a Burgers traveling wave and a Klein--Gordon solution combining
oscillatory and polynomial terms, each in one spatial dimension.
These examples test recovery of distinct space--time structures under
linear and nonlinear governing equations.

\subsubsection{One-dimensional diffusion equation}
\label{sec:experiment-diffusion}

We consider a one-dimensional diffusion equation with a prescribed
source term~\cite{majumdar2023symbolic} on
$(x,t)\in(-1,1)\times(0,1]$,
\begin{equation}
u_t-u_{xx}=(\pi^2-1)e^{-t}\sin(\pi x),
\label{eq:experiment-diffusion-pde}
\end{equation}
subject to homogeneous Dirichlet boundary conditions and a sinusoidal
initial condition,
\begin{equation}
\begin{aligned}
u(-1,t)=u(1,t)&=0,\\[6pt]
u(x,0)&=\sin(\pi x).
\end{aligned}
\label{eq:experiment-diffusion-conditions}
\end{equation}
The reference solution is
\begin{equation}
u_\star(x,t)=e^{-t}\sin(\pi x).
\label{eq:experiment-diffusion-reference}
\end{equation}
The reference solution separates exponential temporal decay from a
sinusoidal spatial factor, providing a direct setting for examining
structure identification and coefficient refinement.

The PINN has four hidden layers of 50 $\tanh$ neurons and 7,851 trainable
parameters, comparable to the 7,963 parameters of PISN. Both models use
double precision, the same training sample counts, and equation, boundary
and initial loss weights of $1:100:100$. Training comprises 20,000 Adam
iterations with a learning rate of $10^{-3}$, followed by at most 5,000
L-BFGS iterations.

Training is performed with $N_r=7{,}500$ interior points, 600 boundary
points and 600 initial points. Symbolic search is run on $N_T=3{,}000$
teacher samples with 16 populations and a maximum expression size of 25,
and physical refinement uses 5,000 interior points. The test set
comprises $N_{\mathrm{test}}=10{,}000$ interior points and 1,000 points for each
constraint group.

As shown in Table~\ref{tab:diffusion-comparison}, DeSyR attains a mean
relative $L^2$ error of $1.94\times10^{-15}$,
compared with $5.45\times10^{-5}$ for the PINN,
$1.75\times10^{-5}$ for PINN--SR and $1.09\times10^{-4}$ for PISN.
Its mean maximum absolute error is $3.11\times10^{-15}$, and both its
PDE and constraint residuals are at the level of double-precision
roundoff.
All five DeSyR runs yield numerically equivalent final expressions and
identical test metrics.

\begin{table}[htbp]
\centering
\caption{Solution errors and physical residuals for the one-dimensional diffusion problem, reported as mean $\pm$ sample standard deviation over five repetitions.  Here $E_{\mathrm{rel}}$ and $E_\infty$ denote the relative $L^2$ and maximum absolute errors, and $e_r$ and $e_c$ the PDE and aggregate constraint residual RMS values.}
\label{tab:diffusion-comparison}
\begingroup
\small
\setlength{\tabcolsep}{5pt}
\renewcommand{\arraystretch}{1.2}
\par\noindent\textbf{(a) Solution errors}\par\smallskip
\begin{tabular}{@{}lcc@{}}
\toprule
Method & $E_{\mathrm{rel}}$ & $E_\infty$ \\
\midrule
PINN & $5.45\!\times\!10^{-5}\!\pm\!1.17\!\times\!10^{-5}$ & $9.11\!\times\!10^{-5}\!\pm\!3.07\!\times\!10^{-5}$ \\
PINN--SR & $1.75\!\times\!10^{-5}\!\pm\!1.70\!\times\!10^{-5}$ & $1.81\!\times\!10^{-5}\!\pm\!2.03\!\times\!10^{-5}$ \\
PISN & $1.09\!\times\!10^{-4}\!\pm\!7.76\!\times\!10^{-6}$ & $1.44\!\times\!10^{-4}\!\pm\!2.75\!\times\!10^{-5}$ \\
DeSyR & $1.94\!\times\!10^{-15}\!\pm\!0$ & $3.11\!\times\!10^{-15}\!\pm\!0$ \\
\bottomrule
\end{tabular}

\par\medskip
\par\noindent\textbf{(b) Physical residuals}\par\smallskip
\begin{tabular}{@{}lcc@{}}
\toprule
Method & $e_r$ & $e_c$ \\
\midrule
PINN & $1.24\!\times\!10^{-3}\!\pm\!1.58\!\times\!10^{-4}$ & $3.02\!\times\!10^{-5}\!\pm\!9.04\!\times\!10^{-6}$ \\
PINN--SR & $1.23\!\times\!10^{-4}\!\pm\!1.56\!\times\!10^{-4}$ & $5.30\!\times\!10^{-6}\!\pm\!7.49\!\times\!10^{-6}$ \\
PISN & $1.18\!\times\!10^{-3}\!\pm\!1.95\!\times\!10^{-4}$ & $3.29\!\times\!10^{-5}\!\pm\!1.44\!\times\!10^{-5}$ \\
DeSyR & $9.88\!\times\!10^{-15}\!\pm\!0$ & $1.90\!\times\!10^{-15}\!\pm\!0$ \\
\bottomrule
\end{tabular}
\endgroup
\end{table}

\begin{samepage}
DeSyR recovers the expression
\begin{equation*}
\widehat u(x,t)=e^{-t}\sin(3.14159265358979\,x),
\end{equation*}
which matches the reference functional form, with the spatial frequency
recovered to numerical precision.
\par\end{samepage}

The representative PINN--SR expression captures the exponential decay
and spatial sine factor, while differing from the reference through a
perturbed frequency and an additional modulation in the exponent.

The corresponding DeSyR candidate belongs to the two-parameter family
\begin{equation}
s(x,t;\kappa,\beta)
=\bigl[\sin(\kappa x)+\beta\bigr]
\exp\!\bigl[-t+\beta\exp(\tanh x)\bigr].
\label{eq:diffusion-refinement-candidate}
\end{equation}
This family represents the reference solution at $(\kappa,\beta)=(\pi,0)$:
$\kappa$ controls the spatial frequency and the shared parameter $\beta$
introduces a redundant perturbation, while the temporal decay rate is
fixed at unity. The initial coefficients are
$\kappa^{(0)}\approx3.14159304$ and
$\beta^{(0)}\approx7.88487\times10^{-6}$. Physical refinement gives
$\widehat{\kappa}=3.141592653589793$ and
$\widehat{\beta}\approx-3.52616\times10^{-18}$, close to the
representing values $(\pi,0)$.

For the parameterized expression~\eqref{eq:diffusion-refinement-candidate},
correcting the frequency and suppressing the perturbation reduce the PDE
residual RMS from $4.76\times10^{-5}$ to $9.88\times10^{-15}$ and the
relative $L^2$ error from $1.73\times10^{-5}$ to $1.94\times10^{-15}$.
The resulting solution and residual errors are well below those of the
PINN teacher.

Figure~\ref{fig:diffusion-comparison} complements the coefficient analysis
by showing the space--time distribution of the errors. The PINN--SR error
forms a regular spatial pattern that varies over time, alongside the
frequency perturbation and additional modulation in its recovered
expression. DeSyR reproduces the separable reference profile with errors
near the floating-point level throughout the plotted domain.

\begin{figure}[htbp]
\centering
\includegraphics[width=\linewidth]{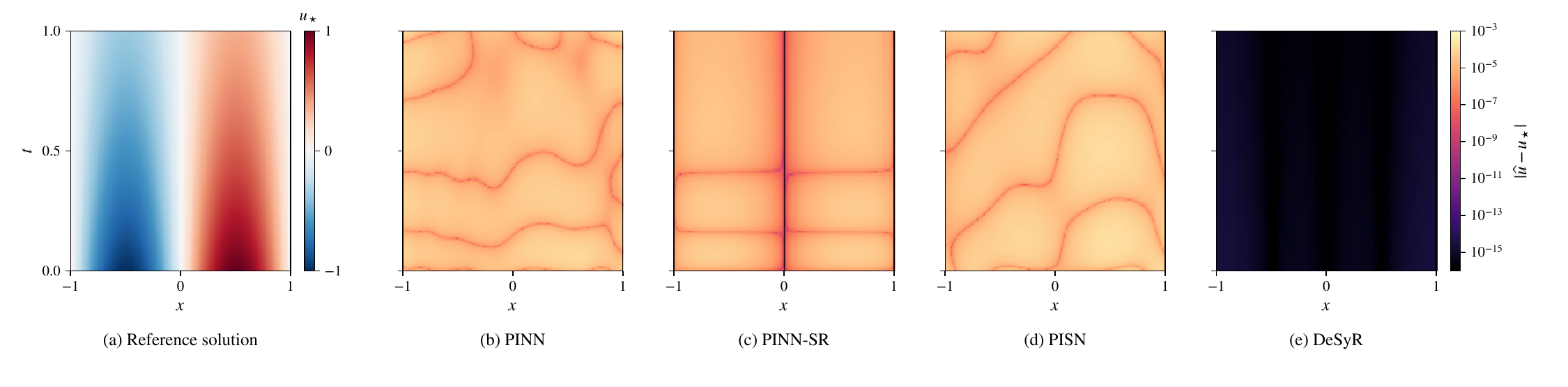}
\caption{Representative solution and absolute-error fields for the
diffusion problem: (a) reference solution; (b)--(e) errors for PINN,
PINN--SR, PISN and DeSyR}
\label{fig:diffusion-comparison}
\end{figure}

\subsubsection{One-dimensional viscous Burgers equation}
\label{sec:experiment-burgers}

Burgers traveling waves provide a test of moving transition
layers~\cite{majumdar2022physics}. We consider the one-dimensional
viscous Burgers equation on
$(x,t)\in(-1,1)\times(0,1]$,
\begin{equation}
u_t+u\,u_x-\nu u_{xx}=0,
\qquad \nu=0.05.
\label{eq:experiment-burgers-pde}
\end{equation}
We use the traveling-wave reference solution
\begin{equation}
u_\star(x,t)
=c-\alpha\tanh\!\left(\frac{\alpha(x-ct-x_0)}{2\nu}\right),
\label{eq:experiment-burgers-reference}
\end{equation}
where $c=\alpha=0.5$ and $x_0=0$. Here, $c$ is the propagation speed,
$\alpha$ is half the jump between the asymptotic states, and $x_0$ is the
initial center of the transition. The inverse width is $\alpha/(2\nu)$.
The boundary and initial conditions are prescribed by the corresponding
traces,
\begin{equation}
\begin{aligned}
u(\pm1,t)&=u_\star(\pm1,t),\\[6pt]
u(x,0)&=u_\star(x,0).
\end{aligned}
\label{eq:experiment-burgers-conditions}
\end{equation}
Unlike the separable diffusion solution, this profile is organized by the
traveling coordinate $x-ct-x_0$. The recovered coefficients determine the
propagation speed and transition-layer gradient, with inverse width
$\alpha/(2\nu)$. The viscosity is prescribed, and refinement adjusts only
the expression coefficients.

The PINN architecture, numerical precision, loss weights and optimization
budgets are the same as in Section~\ref{sec:experiment-diffusion}.
Training is performed with 4,000 interior points, 800 boundary points and
1,000 initial points. The PISN network has a comparable parameter count
(7,963 versus 7,851 for the PINN) and shares the sampling sizes, loss
weights and optimization budgets of the PINN.

Symbolic search is run on $N_T=2{,}000$ teacher samples with 16
populations and a maximum expression size of 30, and physical refinement
uses 5,000 interior points. The test set comprises $N_{\mathrm{test}}=10{,}000$
interior points and 1,000 points per constraint group.

As Table~\ref{tab:burgers-comparison} shows, PINN--SR reduces the mean
relative $L^2$ error of the PINN from
$2.34\times10^{-4}$ to $1.01\times10^{-4}$ and its PDE residual RMS
from $1.09\times10^{-3}$ to $1.84\times10^{-4}$. Its constraint
residual RMS, however, increases from $3.16\times10^{-5}$ to
$6.37\times10^{-5}$. PISN attains a mean relative error of
$1.60\times10^{-4}$, with a PDE residual RMS of $1.12\times10^{-3}$.
DeSyR returns the same expression in all five runs. Its solution and
constraint errors are zero in double-precision evaluation, while its PDE
residual RMS is $1.14\times10^{-16}$.

\begin{table}[htbp]
\centering
\caption{Solution errors and physical residuals for the Burgers traveling wave, reported as mean $\pm$ sample standard deviation over five repetitions. Here $E_{\mathrm{rel}}$ and $E_\infty$ denote the relative $L^2$ and maximum absolute errors, and $e_r$ and $e_c$ the PDE and aggregate constraint residual RMS values.}
\label{tab:burgers-comparison}
\begingroup
\small
\setlength{\tabcolsep}{5pt}
\renewcommand{\arraystretch}{1.2}
\par\noindent\textbf{(a) Solution errors}\par\smallskip
\begin{tabular}{@{}lcc@{}}
\toprule
Method & $E_{\mathrm{rel}}$ & $E_\infty$ \\
\midrule
PINN & $2.34\!\times\!10^{-4}\!\pm\!8.14\!\times\!10^{-5}$ & $6.55\!\times\!10^{-4}\!\pm\!1.49\!\times\!10^{-4}$ \\
PINN--SR & $1.01\!\times\!10^{-4}\!\pm\!5.49\!\times\!10^{-5}$ & $2.10\!\times\!10^{-4}\!\pm\!1.83\!\times\!10^{-4}$ \\
PISN & $1.60\!\times\!10^{-4}\!\pm\!8.38\!\times\!10^{-5}$ & $3.67\!\times\!10^{-4}\!\pm\!1.79\!\times\!10^{-4}$ \\
DeSyR & $0\!\pm\!0$ & $0\!\pm\!0$ \\
\bottomrule
\end{tabular}

\par\medskip
\par\noindent\textbf{(b) Physical residuals}\par\smallskip
\begin{tabular}{@{}lcc@{}}
\toprule
Method & $e_r$ & $e_c$ \\
\midrule
PINN & $1.09\!\times\!10^{-3}\!\pm\!2.09\!\times\!10^{-4}$ & $3.16\!\times\!10^{-5}\!\pm\!7.84\!\times\!10^{-6}$ \\
PINN--SR & $1.84\!\times\!10^{-4}\!\pm\!1.33\!\times\!10^{-4}$ & $6.37\!\times\!10^{-5}\!\pm\!2.82\!\times\!10^{-5}$ \\
PISN & $1.12\!\times\!10^{-3}\!\pm\!5.08\!\times\!10^{-4}$ & $3.31\!\times\!10^{-5}\!\pm\!1.24\!\times\!10^{-5}$ \\
DeSyR & $1.14\!\times\!10^{-16}\!\pm\!0$ & $0\!\pm\!0$ \\
\bottomrule
\end{tabular}
\endgroup
\end{table}

DeSyR recovers
\begin{equation*}
\widehat u(x,t)=0.5+0.5\tanh(2.5t-5x),
\end{equation*}
which is algebraically equivalent to
\eqref{eq:experiment-burgers-reference}. Its affine argument recovers the traveling coordinate, with propagation
speed $c=0.5$, amplitude parameter $\alpha=0.5$, initial center $x_0=0$
and inverse width 5, consistent with $\alpha/(2\nu)$ for the prescribed
viscosity. In contrast, the representative
PINN--SR expression retains an additional nonlinear space--time
correction within the $\tanh$ argument, rather than reproducing the
affine traveling coordinate of the reference wave.

\begin{samepage}
Figure~\ref{fig:burgers-comparison} shows the translating transition
layer and the associated error fields. The baseline errors form oblique
bands along the direction of wave propagation, contrasting with the
separable profile in the diffusion example. DeSyR reproduces the moving
profile to double precision across the plotted grid, complementing the
recovered speed and inverse width identified above.
\par\end{samepage}
\begin{figure}[htbp]
\centering
\includegraphics[width=\linewidth]{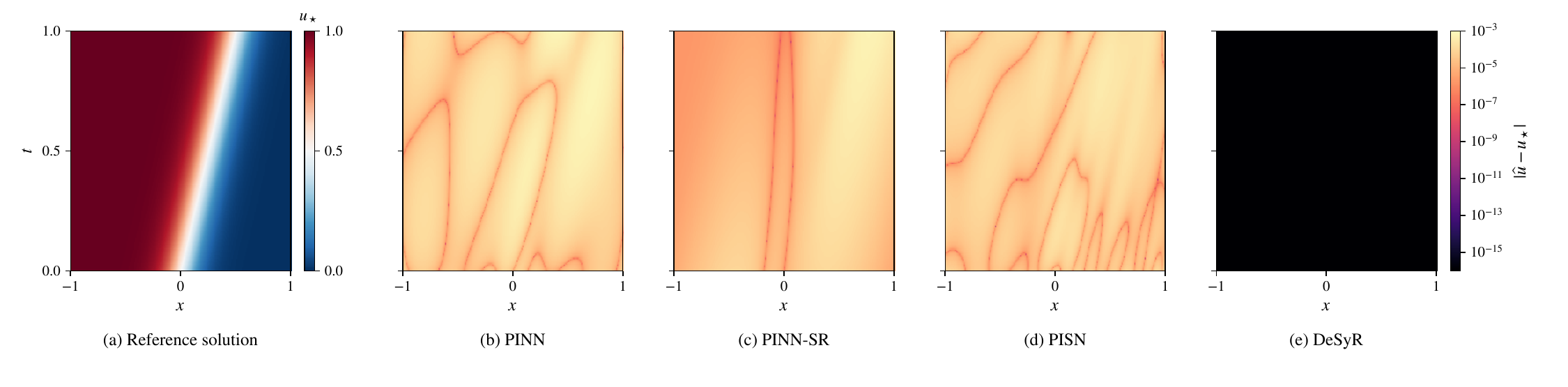}
\caption{Representative solution and absolute-error fields for the
Burgers traveling wave: (a) reference solution; (b)--(e) errors for PINN,
PINN--SR, PISN and DeSyR}
\label{fig:burgers-comparison}
\end{figure}

\subsubsection{Klein--Gordon equation}
\label{sec:experiment-klein-gordon}

The nonlinear Klein--Gordon equation combines second-order wave dynamics
with a nonlinear reaction term. We consider the forced cubic
problem~\cite{wang2021understanding} on $(x,t)\in(0,1)\times(0,1]$,
\begin{equation}
u_{tt}-u_{xx}+u^3=f(x,t),
\label{eq:experiment-klein-gordon-pde}
\end{equation}
with the prescribed source
\begin{equation}
\begin{aligned}
f(x,t)
={}&-25\pi^2x\cos(5\pi t)+6x^3t-6xt^3\\
&+\bigl[x\cos(5\pi t)+x^3t^3\bigr]^3.
\end{aligned}
\label{eq:experiment-klein-gordon-source}
\end{equation}
The boundary and initial conditions are
\begin{alignat}{2}
u(0,t)&=0, &\qquad u(1,t)&=\cos(5\pi t)+t^3,
\label{eq:experiment-klein-gordon-boundary}\\[6pt]
u(x,0)&=x, &\qquad u_t(x,0)&=0.
\label{eq:experiment-klein-gordon-initial}
\end{alignat}
The corresponding reference solution is
\begin{equation}
u_\star(x,t)=x\cos(5\pi t)+x^3t^3.
\label{eq:experiment-klein-gordon-reference}
\end{equation}
The reference solution combines temporal oscillations with a space--time
polynomial. The cubic term couples these contributions, so accurate
recovery requires both the additive structure and coefficients consistent
with the nonlinear equation.

The PINN and PISN architectures and the two-stage training settings
are the same as in Section~\ref{sec:experiment-diffusion}. Both networks
are trained with $N_r=5{,}000$ interior points, 2,000 boundary points and
2,000 initial points, with equal weights for the equation, boundary,
initial-value and initial-derivative loss terms. Symbolic search is run
on $N_T=1{,}000$ teacher samples with 16 populations and a maximum
expression size of 30, and physical refinement uses 4,000
interior points. The test set comprises $N_{\mathrm{test}}=5{,}000$ interior points and
1,000 points per constraint group; the aggregate constraint residual
includes both the initial values and the initial time derivatives.

The results in Table~\ref{tab:klein-gordon-comparison} show that PINN--SR
reduces the mean relative $L^2$ error from the PINN's
$1.64\times10^{-2}$ to $1.07\times10^{-2}$ and lowers the aggregate
constraint residual, while its mean PDE residual RMS increases from
$1.79\times10^{-2}$ to $2.01\times10^{-1}$.
PISN has a similar mean solution error of
$1.04\times10^{-2}$, but its PDE residual RMS, $1.65\times10^{-2}$,
remains comparable to the PINN's. Thus, similar solution errors can
mask substantial differences in equation satisfaction.

DeSyR achieves a mean relative $L^2$ error of $1.85\times10^{-14}$ and
a mean maximum absolute error of $2.98\times10^{-14}$, together with PDE
and constraint residual RMS values of $2.02\times10^{-12}$ and
$6.89\times10^{-15}$, respectively. All five runs produce the same final
expression and identical test metrics. These results combine high solution
accuracy with small equation and constraint residuals.

\begin{table}[htbp]
\centering
\caption{Solution errors and physical residuals for the Klein--Gordon problem, reported as mean $\pm$ sample standard deviation over five repetitions.  Here $E_{\mathrm{rel}}$ and $E_\infty$ denote the relative $L^2$ and maximum absolute errors, and $e_r$ and $e_c$ the PDE and aggregate constraint residual RMS values.}
\label{tab:klein-gordon-comparison}
\begingroup
\small
\setlength{\tabcolsep}{5pt}
\renewcommand{\arraystretch}{1.2}
\par\noindent\textbf{(a) Solution errors}\par\smallskip
\begin{tabular}{@{}lcc@{}}
\toprule
Method & $E_{\mathrm{rel}}$ & $E_\infty$ \\
\midrule
PINN & $1.64\!\times\!10^{-2}\!\pm\!5.14\!\times\!10^{-3}$ & $3.86\!\times\!10^{-2}\!\pm\!1.67\!\times\!10^{-2}$ \\
PINN--SR & $1.07\!\times\!10^{-2}\!\pm\!6.92\!\times\!10^{-3}$ & $2.17\!\times\!10^{-2}\!\pm\!1.59\!\times\!10^{-2}$ \\
PISN & $1.04\!\times\!10^{-2}\!\pm\!4.49\!\times\!10^{-3}$ & $1.74\!\times\!10^{-2}\!\pm\!6.94\!\times\!10^{-3}$ \\
DeSyR & $1.85\!\times\!10^{-14}\!\pm\!0$ & $2.98\!\times\!10^{-14}\!\pm\!0$ \\
\bottomrule
\end{tabular}

\par\medskip
\par\noindent\textbf{(b) Physical residuals}\par\smallskip
\begin{tabular}{@{}lcc@{}}
\toprule
Method & $e_r$ & $e_c$ \\
\midrule
PINN & $1.79\!\times\!10^{-2}\!\pm\!3.25\!\times\!10^{-3}$ & $8.02\!\times\!10^{-3}\!\pm\!3.08\!\times\!10^{-3}$ \\
PINN--SR & $2.01\!\times\!10^{-1}\!\pm\!8.54\!\times\!10^{-2}$ & $3.68\!\times\!10^{-3}\!\pm\!3.82\!\times\!10^{-3}$ \\
PISN & $1.65\!\times\!10^{-2}\!\pm\!3.48\!\times\!10^{-3}$ & $4.75\!\times\!10^{-3}\!\pm\!1.88\!\times\!10^{-3}$ \\
DeSyR & $2.02\!\times\!10^{-12}\!\pm\!0$ & $6.89\!\times\!10^{-15}\!\pm\!0$ \\
\bottomrule
\end{tabular}
\endgroup
\end{table}

DeSyR recovers
\begin{equation*}
\widehat u(x,t)=x\cos(15.707963267949\,t)+x^3t^3.
\end{equation*}
The representative PINN--SR expression shares this additive structure,
so the comparison highlights the accuracy of the recovered frequency.

The corresponding DeSyR candidate is
$s(x,t;\omega)=x\cos(\omega t)+x^3t^3$, with a single free frequency
and the powers and other coefficients fixed. Its initial value
$\omega^{(0)}\approx15.70622858$ differs from the reference frequency
$5\pi$, whereas the refined value
$\widehat\omega=15.707963267948966$ agrees with it to double-precision
accuracy.

\begin{samepage}
\noindent The frequency enters the temporal second derivative as
$s_{tt}=-\omega^2x\cos(\omega t)+6x^3t$.
Its estimation therefore affects the oscillatory phase, the second-derivative
amplitude and the cubic reaction term. Physical refinement determines the
frequency by jointly accounting for these effects in the equation residual
and the prescribed constraints.
\par\end{samepage}

The temporal bands in the representative PINN--SR error field in
Figure~\ref{fig:klein-gordon-comparison} are consistent with this
frequency mismatch. DeSyR reproduces both the oscillatory and polynomial
contributions with much smaller errors across the plotted domain.

\begin{figure}[htbp]
\centering
\includegraphics[width=\linewidth]{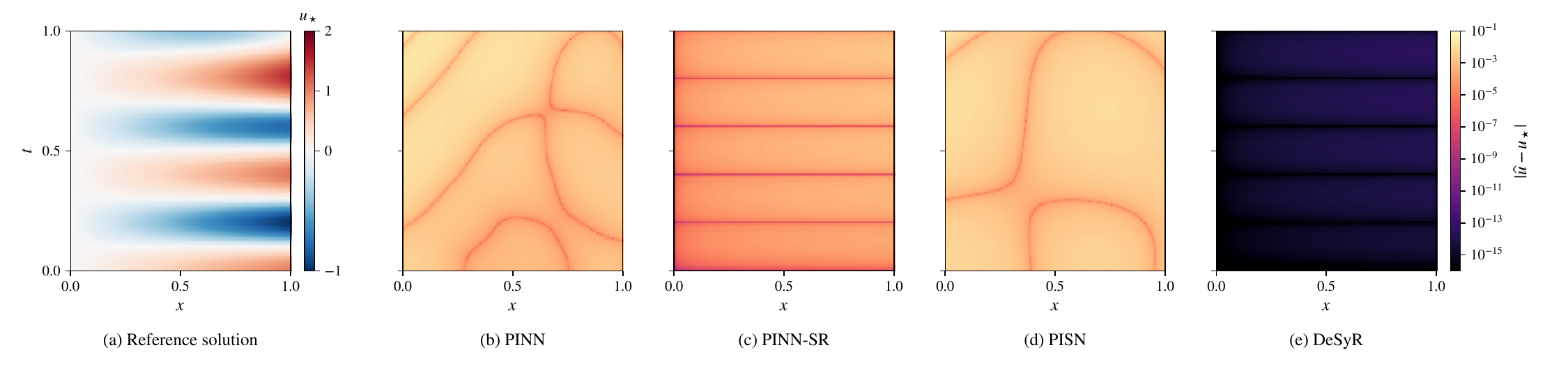}
\caption{Representative solution and absolute-error fields for the
Klein--Gordon problem: (a) reference solution; (b)--(e) errors for PINN,
PINN--SR, PISN and DeSyR}
\label{fig:klein-gordon-comparison}
\end{figure}

\subsection{Multidimensional elliptic problems}
\label{sec:multidimensional-elliptic-problems}

We next examine multidimensional product structures using two elliptic
problems. The two-dimensional Helmholtz example involves higher-frequency
spatial oscillations and a shared refinement parameter, while the
three-dimensional Poisson example extends the product structure to three
independently refined frequencies.

\subsubsection{Two-dimensional Helmholtz equation}
\label{sec:experiment-helmholtz}

We consider a two-dimensional Helmholtz equation with an oscillatory
source term~\cite{wang2021understanding} on $\Omega=(-1,1)^2$,
\begin{equation}
\begin{aligned}
\Delta u+u&=(1-32\pi^2)\sin(4\pi x)\sin(4\pi y),\\[6pt]
u|_{\partial\Omega}&=0.
\end{aligned}
\label{eq:experiment-helmholtz-pde}
\end{equation}
The reference solution is
\begin{equation}
u_\star(x,y)=\sin(4\pi x)\sin(4\pi y).
\label{eq:experiment-helmholtz-reference}
\end{equation}
The Helmholtz parameter in the operator is fixed at unity, while the
reference solution contains four sinusoidal periods along each coordinate
direction. This example examines recovery of a two-dimensional oscillatory
product and the precision of its frequency estimate.

The PINN architecture is the same as in
Section~\ref{sec:experiment-diffusion}. The PISN network has a comparable
parameter count (7,963 versus 7,851 for the PINN), and both models are
trained with $N_r=8{,}000$ interior points, 800 boundary points and equal
loss weights. The Adam stage comprises 20,000 iterations for the PINN and
40,000 for PISN, followed by at most 5,000 L-BFGS iterations for each
model. Symbolic search is run on $N_T=2{,}000$ teacher samples with 12
populations and a maximum expression size of 30, and physical refinement
uses 4,000 interior points. The test set comprises $N_{\mathrm{test}}=5{,}000$
interior points and 1,000 points per boundary face.

The PINN has a mean relative $L^2$ error of $6.57\times10^{-2}$
(Table~\ref{tab:helmholtz-comparison}).
PISN reduces the mean relative error to
$1.34\times10^{-2}$ under its larger training budget.
PINN--SR attains a lower mean relative error of $5.298\times10^{-3}$.
Compared with the PINN teacher, it also reduces the aggregate constraint
residual, while increasing the PDE residual RMS to $8.423\times10^{-1}$.
DeSyR combines a mean relative error of $2.97\times10^{-14}$ with a
mean PDE residual RMS of $4.73\times10^{-12}$, improving both solution
accuracy and equation satisfaction.

\begin{table}[htbp]
\centering
\caption{Solution errors and physical residuals for the two-dimensional Helmholtz problem, reported as mean $\pm$ sample standard deviation over five repetitions. Here $E_{\mathrm{rel}}$ and $E_\infty$ denote the relative $L^2$ and maximum absolute errors, and $e_r$ and $e_c$ the PDE and aggregate constraint residual RMS values.}
\label{tab:helmholtz-comparison}
\begingroup
\small
\setlength{\tabcolsep}{5pt}
\renewcommand{\arraystretch}{1.2}
\par\noindent\textbf{(a) Solution errors}\par\smallskip
\begin{tabular}{@{}lcc@{}}
\toprule
Method & $E_{\mathrm{rel}}$ & $E_\infty$ \\
\midrule
PINN & $6.57\!\times\!10^{-2}\!\pm\!1.40\!\times\!10^{-2}$ & $3.58\!\times\!10^{-1}\!\pm\!9.14\!\times\!10^{-2}$ \\
PINN--SR & $5.298\!\times\!10^{-3}\!\pm\!5.495\!\times\!10^{-4}$ & $8.074\!\times\!10^{-3}\!\pm\!1.233\!\times\!10^{-3}$ \\
PISN & $1.34\!\times\!10^{-2}\!\pm\!5.89\!\times\!10^{-3}$ & $4.04\!\times\!10^{-2}\!\pm\!1.84\!\times\!10^{-2}$ \\
DeSyR & $2.97\!\times\!10^{-14}\!\pm\!1.31\!\times\!10^{-14}$ & $3.99\!\times\!10^{-14}\!\pm\!2.12\!\times\!10^{-14}$ \\
\bottomrule
\end{tabular}

\par\medskip
\par\noindent\textbf{(b) Physical residuals}\par\smallskip
\begin{tabular}{@{}lcc@{}}
\toprule
Method & $e_r$ & $e_c$ \\
\midrule
PINN & $2.41\!\times\!10^{-1}\!\pm\!4.51\!\times\!10^{-2}$ & $1.05\!\times\!10^{-1}\!\pm\!2.04\!\times\!10^{-2}$ \\
PINN--SR & $8.423\!\times\!10^{-1}\!\pm\!8.688\!\times\!10^{-2}$ & $4.658\!\times\!10^{-3}\!\pm\!5.080\!\times\!10^{-4}$ \\
PISN & $6.87\!\times\!10^{-2}\!\pm\!3.34\!\times\!10^{-2}$ & $1.47\!\times\!10^{-2}\!\pm\!6.39\!\times\!10^{-3}$ \\
DeSyR & $4.73\!\times\!10^{-12}\!\pm\!2.08\!\times\!10^{-12}$ & $2.25\!\times\!10^{-14}\!\pm\!5.01\!\times\!10^{-15}$ \\
\bottomrule
\end{tabular}
\endgroup
\end{table}

\begin{samepage}
The representative DeSyR expression is
\begin{equation*}
\widehat u(x,y)=\sin(12.5663706143592\,x)
\sin(12.5663706143592\,y).
\end{equation*}
The DeSyR candidate is parameterized as $\sin(x/a)\sin(y/a)$, with one
adjustable parameter shared by the two factors. Refinement brings the
common frequency close to $4\pi$. PINN--SR also identifies the sinusoidal
product structure, but with less accurate frequencies.
\par\end{samepage}

The representative error fields in
Figure~\ref{fig:helmholtz-comparison} show spatially oscillatory errors
for PINN--SR and much smaller errors for DeSyR across the plotted domain.

\begin{figure}[htbp]
\centering
\includegraphics[width=\linewidth]{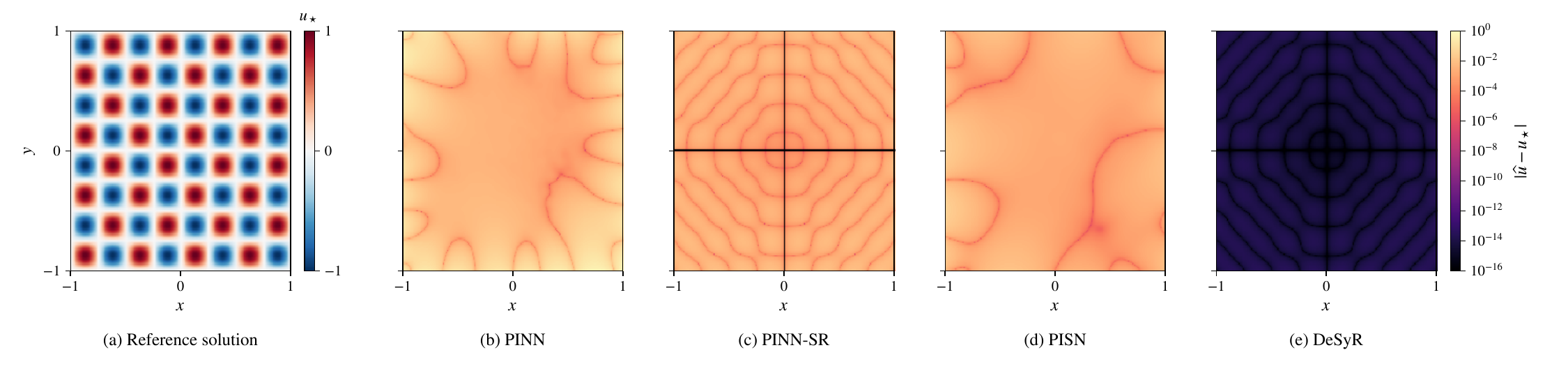}
\caption{Representative solution and absolute-error fields for the
two-dimensional Helmholtz problem: (a) reference solution; (b)--(e)
errors for PINN, PINN--SR, PISN and DeSyR}
\label{fig:helmholtz-comparison}
\end{figure}

\subsubsection{Three-dimensional Poisson equation}
\label{sec:experiment-poisson-3d}

We consider a three-dimensional Poisson equation with a separable
sinusoidal source~\cite{oh2024inherently} on $\Omega=(0,1)^3$,
\begin{equation}
\begin{aligned}
-\Delta u&=3\pi^2\sin(\pi x)\sin(\pi y)\sin(\pi z),\\[6pt]
u|_{\partial\Omega}&=0,
\end{aligned}
\label{eq:experiment-poisson-3d-pde}
\end{equation}
with reference solution
\begin{equation}
u_\star(x,y,z)=\sin(\pi x)\sin(\pi y)\sin(\pi z).
\label{eq:experiment-poisson-3d-reference}
\end{equation}
The solution extends the spatial product structure to three coordinates,
requiring recovery of three sinusoidal factors and their frequencies.

The PINN has the same hidden layers and optimization budget as in
Section~\ref{sec:experiment-diffusion} and is trained with
$N_r=10{,}000$ interior points, 1,200 boundary points and a boundary
weight of 100. The PISN network has a comparable parameter count (7,828
versus 7,901 for the PINN) and shares the sampling sizes, loss weights
and optimization budget of the PINN. Symbolic search is run on
$N_T=4{,}000$ teacher samples with 16 populations and a maximum
expression size of 25. Physical refinement uses 8,000
interior points, and the test set comprises $N_{\mathrm{test}}=10{,}000$ interior points
and 1,000 points per boundary face.

Table~\ref{tab:poisson-3d-comparison} reports a mean relative $L^2$ error
of $2.00\times10^{-4}$ for the PINN and $1.89\times10^{-5}$ for PINN--SR.
PISN yields a mean relative error of $2.47\times10^{-4}$ and a mean PDE
residual RMS of $3.39\times10^{-3}$, both comparable to the PINN's.
DeSyR produces the same expression in all five runs, with a relative
$L^2$ error of $3.52\times10^{-15}$ and PDE and boundary residual RMS
values of $3.55\times10^{-14}$ and $1.15\times10^{-15}$, respectively.

\begin{table}[htbp]
\centering
\caption{Solution errors and physical residuals for the three-dimensional Poisson problem, reported as mean $\pm$ sample standard deviation over five repetitions. Here $E_{\mathrm{rel}}$ and $E_\infty$ denote the relative $L^2$ and maximum absolute errors, and $e_r$ and $e_c$ the PDE and aggregate constraint residual RMS values.}
\label{tab:poisson-3d-comparison}
\begingroup
\small
\setlength{\tabcolsep}{5pt}
\renewcommand{\arraystretch}{1.2}
\par\noindent\textbf{(a) Solution errors}\par\smallskip
\begin{tabular}{@{}lcc@{}}
\toprule
Method & $E_{\mathrm{rel}}$ & $E_\infty$ \\
\midrule
PINN & $2.00\!\times\!10^{-4}\!\pm\!1.81\!\times\!10^{-5}$ & $1.05\!\times\!10^{-3}\!\pm\!2.60\!\times\!10^{-4}$ \\
PINN--SR & $1.89\!\times\!10^{-5}\!\pm\!1.10\!\times\!10^{-5}$ & $2.40\!\times\!10^{-5}\!\pm\!1.75\!\times\!10^{-5}$ \\
PISN & $2.47\!\times\!10^{-4}\!\pm\!5.74\!\times\!10^{-5}$ & $7.33\!\times\!10^{-4}\!\pm\!3.18\!\times\!10^{-4}$ \\
DeSyR & $3.52\!\times\!10^{-15}\!\pm\!0$ & $3.11\!\times\!10^{-15}\!\pm\!0$ \\
\bottomrule
\end{tabular}

\par\medskip
\par\noindent\textbf{(b) Physical residuals}\par\smallskip
\begin{tabular}{@{}lcc@{}}
\toprule
Method & $e_r$ & $e_c$ \\
\midrule
PINN & $3.43\!\times\!10^{-3}\!\pm\!2.94\!\times\!10^{-4}$ & $1.66\!\times\!10^{-4}\!\pm\!1.54\!\times\!10^{-5}$ \\
PINN--SR & $1.93\!\times\!10^{-4}\!\pm\!1.09\!\times\!10^{-4}$ & $6.24\!\times\!10^{-6}\!\pm\!3.45\!\times\!10^{-6}$ \\
PISN & $3.39\!\times\!10^{-3}\!\pm\!1.18\!\times\!10^{-3}$ & $1.85\!\times\!10^{-4}\!\pm\!4.28\!\times\!10^{-5}$ \\
DeSyR & $3.55\!\times\!10^{-14}\!\pm\!0$ & $1.15\!\times\!10^{-15}\!\pm\!0$ \\
\bottomrule
\end{tabular}
\endgroup
\end{table}

The recovered DeSyR expression is
\begin{equation*}
\begin{aligned}
\widehat u(x,y,z)={}&\sin(3.14159265358979\,x)\sin(3.14159265358979\,y)\\
&{}\times\sin(3.14159265358979\,z).
\end{aligned}
\end{equation*}
Both symbolic methods recover the three sinusoidal factors. In the
DeSyR candidate, the three frequency parameters are refined independently
and all agree with $\pi$ to numerical precision; PINN--SR retains less
accurate frequencies. In contrast to the shared parameter in the Helmholtz
example, this candidate recovers the common frequency through independent
coefficient refinement.

On the three orthogonal midplanes $x=0.5$, $y=0.5$ and $z=0.5$,
Figure~\ref{fig:poisson-3d-comparison} shows spatially varying errors for
the baselines and much smaller errors for DeSyR on all three sections.
This cross-sectional comparison complements the full-volume error and
residual measures in Table~\ref{tab:poisson-3d-comparison}.

\begin{figure}[htbp]
\centering
\includegraphics[width=\linewidth]{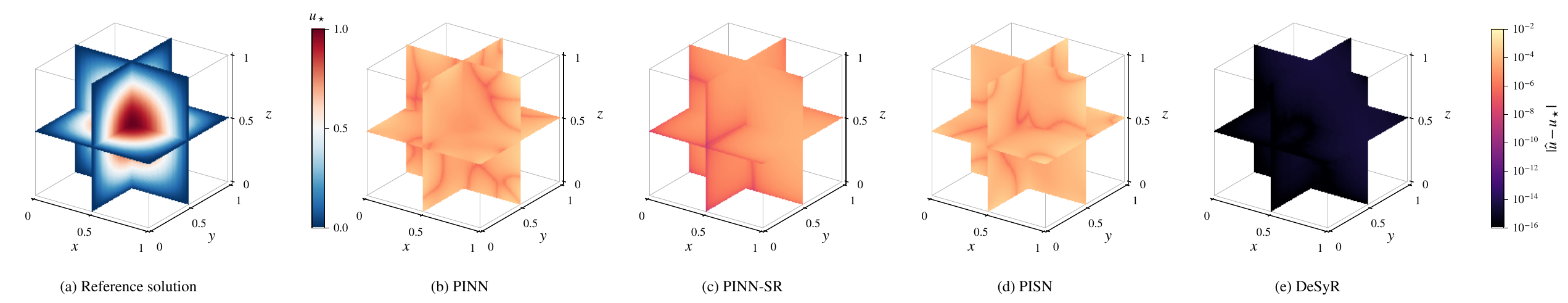}
\caption{Representative solution and absolute-error fields for the
three-dimensional Poisson problem on the orthogonal planes $x=0.5$,
$y=0.5$ and $z=0.5$: (a) reference solution; (b)--(e) errors for PINN,
PINN--SR, PISN and DeSyR}
\label{fig:poisson-3d-comparison}
\end{figure}

\FloatBarrier
\subsection{Incompressible Navier--Stokes equations}
\label{sec:experiment-navier-stokes}

\begin{samepage}
We next consider coupled recovery of two velocity components and pressure
under the incompressible Navier--Stokes equations. The steady Kovasznay
example examines spatial relations among the three fields, while the
unsteady Taylor--Green vortex additionally involves linked temporal decay.
Both are governed by the two-dimensional system
\begin{equation}
\begin{aligned}
u_t+uu_x+vu_y+p_x-\nu(u_{xx}+u_{yy})&=0,\\[6pt]
v_t+uv_x+vv_y+p_y-\nu(v_{xx}+v_{yy})&=0,\\[6pt]
u_x+v_y&=0,
\end{aligned}
\label{eq:experiment-navier-stokes}
\end{equation}
\par\end{samepage}
where $\nu$ is the kinematic viscosity. The time derivatives are omitted
for steady flow. Joint recovery accounts for the derivative relations linking velocity
and pressure through momentum balance and incompressibility.
We report errors separately for $u$, $v$ and $p$, together with the RMS
residuals $e_{\mathrm{mom},x}$, $e_{\mathrm{mom},y}$ and
$e_{\mathrm{div}}$ of the two momentum equations and the continuity equation.
Pressure is evaluated with the gauge fixed by each problem's constraints.

In both examples,
symbolic search is run on $N_T=3{,}000$ teacher samples per field with 16
populations and a maximum expression size of 30, and physical refinement
uses 6,000 interior points. The test set comprises $N_{\mathrm{test}}=10{,}000$
interior points and 1,000 points per boundary or initial constraint group.

\subsubsection{Kovasznay flow}
\label{sec:experiment-kovasznay}

We consider steady Kovasznay
flow~\cite{kovasznay1948laminar,majumdar2022physics,majumdar2023symbolic} on
$\Omega=(-0.5,1)\times(-0.5,1.5)$ at $\mathrm{Re}=20$, with
$\nu=\mathrm{Re}^{-1}$. Its reference solution is
\begin{equation}
\begin{aligned}
u_\star(x,y)&=1-e^{\lambda x}\cos(2\pi y),\\[6pt]
v_\star(x,y)&=\frac{\lambda}{2\pi}e^{\lambda x}\sin(2\pi y),\\[6pt]
p_\star(x,y)&=\frac{1-e^{2\lambda x}}{2},\\[6pt]
\lambda&=\frac{\mathrm{Re}}{2}
-\sqrt{\frac{\mathrm{Re}^2}{4}+4\pi^2}.
\end{aligned}
\label{eq:experiment-kovasznay-reference}
\end{equation}
Velocity values from \eqref{eq:experiment-kovasznay-reference} are
prescribed on the boundary, and $p(0,0)=0$ fixes the pressure constant.
In the reference solution, the velocity components share an exponential
rate, and the pressure rate is twice as large. The transverse velocity
amplitude is related to the same rate through $\lambda/(2\pi)$.
These relations provide specific cross-field features for assessing the
recovered expressions.

The PINN has four hidden layers of 50 neurons and three outputs, and
the PISN network has a comparable parameter count (7,893 versus 7,953 for
the PINN). Both models are trained with $N_r=10{,}000$ interior points and
1,200 boundary points, using loss weights of 1 for the equations and 100
for the velocity constraints and the pressure anchor. The Adam and L-BFGS training stages comprise 20,000
and at most 5,000 iterations, respectively.

PINN--SR reduces the mean relative errors of all three fields compared
with the PINN (Table~\ref{tab:kovasznay-comparison}). Its mean errors
for $u$, $v$ and $p$ are
$7.32\times10^{-5}$, $1.86\times10^{-4}$ and $3.17\times10^{-4}$,
respectively. Its mean momentum and continuity residuals remain between
$2.87\times10^{-4}$ and $7.47\times10^{-4}$
(Table~\ref{tab:kovasznay-physics}). PISN yields mean field errors between $4.40\times10^{-4}$ and
$2.02\times10^{-3}$ under the stated configuration. DeSyR reaches
mean relative errors of $1.77\times10^{-15}$, $3.40\times10^{-15}$ and
$1.72\times10^{-15}$ for the three fields, with mean momentum and
continuity residuals between $1.50\times10^{-16}$ and
$2.22\times10^{-14}$. DeSyR thus combines accurate velocity and pressure fields with small
momentum and continuity residuals.

\begin{table}[!htbp]
\centering
\caption{Fieldwise errors for Kovasznay flow, reported as mean $\pm$ sample standard deviation over five repetitions. Here $E_{\mathrm{rel}}$ and $E_\infty$ denote the relative $L^2$ and maximum absolute errors of the corresponding field.}
\label{tab:kovasznay-comparison}
\begingroup
\small
\setlength{\tabcolsep}{5pt}
\renewcommand{\arraystretch}{1.2}
\par\noindent\textbf{(a) Horizontal velocity $u$}\par\smallskip
\begin{tabular}{@{}lcc@{}}
\toprule
Method & $E_{\mathrm{rel}}$ & $E_\infty$ \\
\midrule
PINN & $2.10\!\times\!10^{-4}\!\pm\!4.10\!\times\!10^{-5}$ & $8.69\!\times\!10^{-4}\!\pm\!1.84\!\times\!10^{-4}$ \\
PINN--SR & $7.32\!\times\!10^{-5}\!\pm\!3.92\!\times\!10^{-5}$ & $3.57\!\times\!10^{-4}\!\pm\!1.37\!\times\!10^{-4}$ \\
PISN & $4.40\!\times\!10^{-4}\!\pm\!1.72\!\times\!10^{-4}$ & $1.63\!\times\!10^{-3}\!\pm\!6.06\!\times\!10^{-4}$ \\
DeSyR & $1.77\!\times\!10^{-15}\!\pm\!0$ & $1.33\!\times\!10^{-14}\!\pm\!0$ \\
\bottomrule
\end{tabular}

\par\medskip
\par\noindent\textbf{(b) Vertical velocity $v$}\par\smallskip
\begin{tabular}{@{}lcc@{}}
\toprule
Method & $E_{\mathrm{rel}}$ & $E_\infty$ \\
\midrule
PINN & $8.91\!\times\!10^{-4}\!\pm\!1.78\!\times\!10^{-4}$ & $5.68\!\times\!10^{-4}\!\pm\!1.14\!\times\!10^{-4}$ \\
PINN--SR & $1.86\!\times\!10^{-4}\!\pm\!9.74\!\times\!10^{-5}$ & $1.59\!\times\!10^{-4}\!\pm\!8.87\!\times\!10^{-5}$ \\
PISN & $2.02\!\times\!10^{-3}\!\pm\!8.11\!\times\!10^{-4}$ & $1.27\!\times\!10^{-3}\!\pm\!5.41\!\times\!10^{-4}$ \\
DeSyR & $3.40\!\times\!10^{-15}\!\pm\!0$ & $3.72\!\times\!10^{-15}\!\pm\!0$ \\
\bottomrule
\end{tabular}

\par\medskip
\par\noindent\textbf{(c) Pressure $p$}\par\smallskip
\begin{tabular}{@{}lcc@{}}
\toprule
Method & $E_{\mathrm{rel}}$ & $E_\infty$ \\
\midrule
PINN & $5.56\!\times\!10^{-4}\!\pm\!1.54\!\times\!10^{-4}$ & $3.10\!\times\!10^{-3}\!\pm\!5.88\!\times\!10^{-4}$ \\
PINN--SR & $3.17\!\times\!10^{-4}\!\pm\!1.73\!\times\!10^{-4}$ & $4.43\!\times\!10^{-4}\!\pm\!3.28\!\times\!10^{-4}$ \\
PISN & $1.02\!\times\!10^{-3}\!\pm\!4.04\!\times\!10^{-4}$ & $5.63\!\times\!10^{-3}\!\pm\!2.52\!\times\!10^{-3}$ \\
DeSyR & $1.72\!\times\!10^{-15}\!\pm\!0$ & $5.77\!\times\!10^{-15}\!\pm\!0$ \\
\bottomrule
\end{tabular}

\endgroup
\end{table}
\begin{table}[!htbp]
\centering
\caption{Physical residuals for Kovasznay flow, reported as mean $\pm$ sample standard deviation over five repetitions. Here $e_{\mathrm{mom},x}$, $e_{\mathrm{mom},y}$, $e_{\mathrm{div}}$ and $e_c$ denote the RMS residuals of the two momentum equations, the continuity equation and the aggregate constraints.}
\label{tab:kovasznay-physics}
\begingroup
\small
\setlength{\tabcolsep}{5pt}
\renewcommand{\arraystretch}{1.2}
\par\noindent\textbf{(a) Momentum residuals}\par\smallskip
\begin{tabular}{@{}lcc@{}}
\toprule
Method & $e_{\mathrm{mom},x}$ & $e_{\mathrm{mom},y}$ \\
\midrule
PINN & $3.06\!\times\!10^{-3}\!\pm\!9.02\!\times\!10^{-4}$ & $2.46\!\times\!10^{-3}\!\pm\!7.48\!\times\!10^{-4}$ \\
PINN--SR & $7.47\!\times\!10^{-4}\!\pm\!2.54\!\times\!10^{-4}$ & $3.19\!\times\!10^{-4}\!\pm\!3.58\!\times\!10^{-4}$ \\
PISN & $3.55\!\times\!10^{-3}\!\pm\!9.50\!\times\!10^{-4}$ & $3.56\!\times\!10^{-3}\!\pm\!9.34\!\times\!10^{-4}$ \\
DeSyR & $2.22\!\times\!10^{-14}\!\pm\!0$ & $3.37\!\times\!10^{-16}\!\pm\!0$ \\
\bottomrule
\end{tabular}

\par\medskip
\par\noindent\textbf{(b) Continuity and constraint residuals}\par\smallskip
\begin{tabular}{@{}lcc@{}}
\toprule
Method & $e_{\mathrm{div}}$ & $e_c$ \\
\midrule
PINN & $2.43\!\times\!10^{-3}\!\pm\!4.93\!\times\!10^{-4}$ & $1.31\!\times\!10^{-4}\!\pm\!3.41\!\times\!10^{-5}$ \\
PINN--SR & $2.87\!\times\!10^{-4}\!\pm\!1.62\!\times\!10^{-4}$ & $9.34\!\times\!10^{-5}\!\pm\!4.96\!\times\!10^{-5}$ \\
PISN & $3.93\!\times\!10^{-3}\!\pm\!1.30\!\times\!10^{-3}$ & $1.92\!\times\!10^{-4}\!\pm\!5.66\!\times\!10^{-5}$ \\
DeSyR & $1.50\!\times\!10^{-16}\!\pm\!0$ & $2.41\!\times\!10^{-15}\!\pm\!0$ \\
\bottomrule
\end{tabular}

\endgroup
\end{table}

\begin{samepage}
The representative DeSyR fields are
\begin{equation*}
\begin{aligned}
\widehat u(x,y)&=1-e^{-1.8100981x}\cos(6.2831853y),\\[6pt]
\widehat v(x,y)&=-0.28808606e^{-1.8100981x}\sin(6.2831853y),\\[6pt]
\widehat p(x,y)&=0.5-0.5e^{-3.6201962x},
\end{aligned}
\end{equation*}
with coefficients rounded for display. Both symbolic methods identify
the exponential and trigonometric factors. The DeSyR fields also recover
the coefficient relations of the reference solution to numerical precision:
the velocity components share the spatial exponential rate, the pressure
rate is twice that value, and the transverse velocity amplitude is
$\lambda/(2\pi)$.
\par\end{samepage}

Figure~\ref{fig:kovasznay-comparison} shows horizontal bands in the
PINN--SR vertical-velocity error, alongside much smaller DeSyR errors in
all three fields. The spatial comparison complements the fieldwise errors
in Table~\ref{tab:kovasznay-comparison} and the equation residuals in
Table~\ref{tab:kovasznay-physics}.
\begin{figure}[htbp]
\centering
\includegraphics[width=\linewidth]{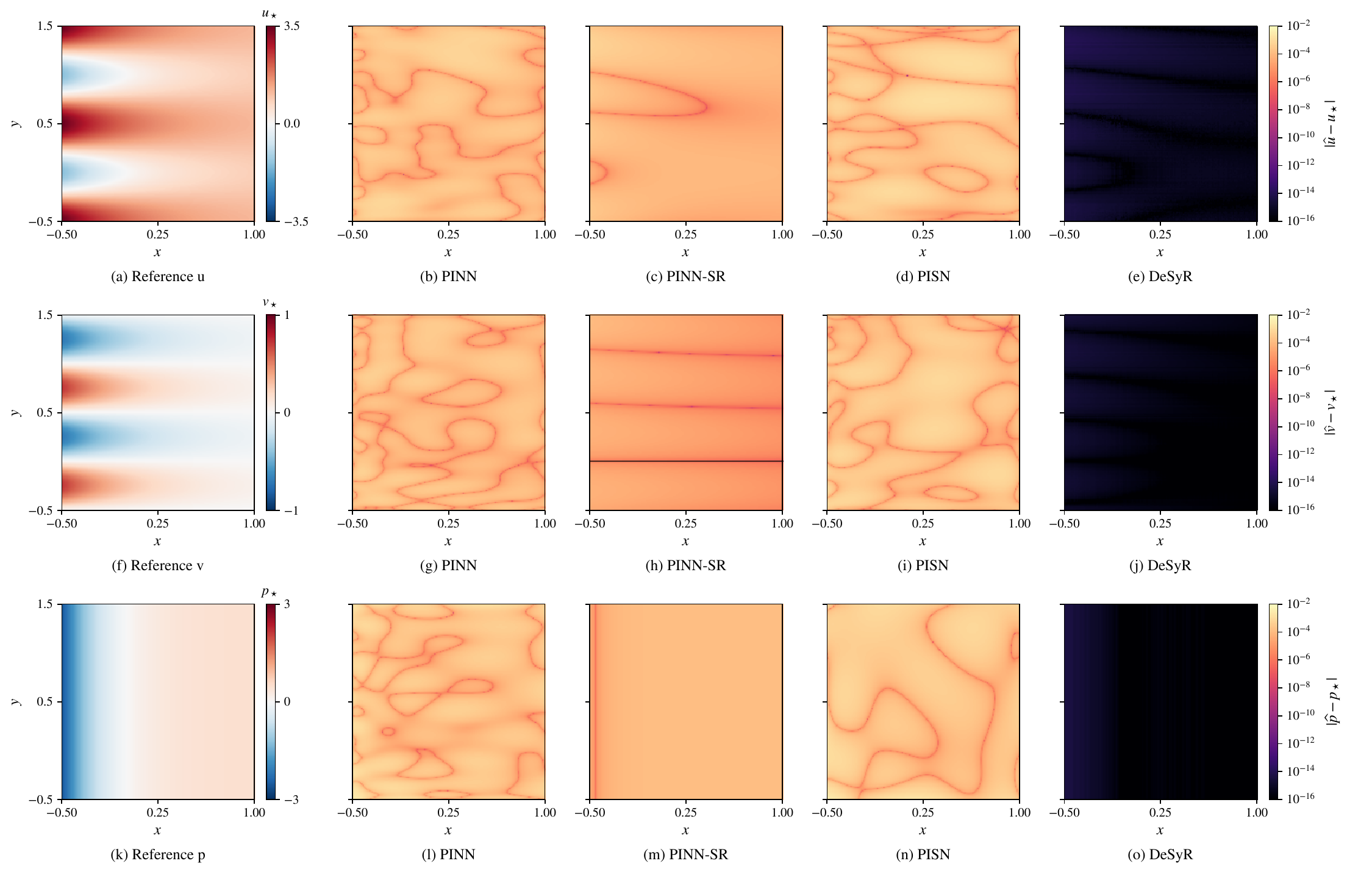}
\caption{Representative fields and absolute errors for Kovasznay flow
at $\mathrm{Re}=20$. Rows show $u$, $v$ and $p$; columns show the reference
field and absolute errors for PINN, PINN--SR, PISN and DeSyR}
\label{fig:kovasznay-comparison}
\end{figure}

\FloatBarrier
\subsubsection{Taylor--Green vortex}
\label{sec:experiment-taylor-green}

The Taylor--Green vortex provides an unsteady coupled-flow
problem~\cite{taylor1937mechanism,majumdar2023symbolic}. On
$(x,y,t)\in(0,2)^2\times(0,1]$, with $\nu=0.01$, the reference fields are
\begin{equation}
\begin{aligned}
u_\star(x,y,t)&=-\cos(\pi x)\sin(\pi y)e^{-2\pi^2\nu t},\\[6pt]
v_\star(x,y,t)&=\sin(\pi x)\cos(\pi y)e^{-2\pi^2\nu t},\\[6pt]
p_\star(x,y,t)&=-\frac{1}{4}[\cos(2\pi x)+\cos(2\pi y)]e^{-4\pi^2\nu t}.
\end{aligned}
\label{eq:experiment-taylor-green-reference}
\end{equation}
The initial conditions prescribe $u$ and $v$ at $t=0$, and Dirichlet
values from \eqref{eq:experiment-taylor-green-reference} are imposed
for all three fields on the spatial boundary. The pressure boundary
values fix its time-dependent additive constant. The velocity fields
share the temporal decay rate $2\pi^2\nu$, while pressure decays twice
as fast. Recovery must therefore capture the spatial factors and the
related temporal rates across fields.

The PINN has the same hidden layers as in the Kovasznay example,
with time as an additional input, and PISN retains the same
symbolic-layer architecture. Both models are trained with $N_r=2{,}500$
interior points, 2,500 boundary points and 2,500 initial points. Each
momentum and continuity residual has weight 1, and each boundary and
initial-value constraint has weight 100. The two-stage training settings
are those of the Kovasznay example. Using the subexpression recombination
described in Section~\ref{sec:coupled-recovery}, DeSyR combines spatial
and temporal factors found by symbolic search to construct space--time
candidates.

The PINN has mean relative errors of $1.20\times10^{-3}$,
$1.29\times10^{-3}$ and $3.44\times10^{-3}$ for $u$, $v$ and $p$
(Table~\ref{tab:taylor-green-comparison}).
PINN--SR reduces the mean errors in $v$ and $p$, while increasing that in
$u$ to $1.49\times10^{-3}$. Its mean $x$-momentum residual RMS rises from
$4.78\times10^{-3}$ to $7.03\times10^{-3}$; the $y$-momentum and
continuity residuals decrease, whereas the aggregate constraint residual
increases (Table~\ref{tab:taylor-green-physics}).
PISN has mean velocity errors of order
$10^{-3}$ and a mean pressure error of $5.32\times10^{-3}$.
DeSyR yields mean relative errors of $4.84\times10^{-15}$,
$5.11\times10^{-15}$ and $4.23\times10^{-15}$. These smaller fieldwise
errors are accompanied by lower residuals for both momentum equations,
continuity and the prescribed constraints.

\begin{table}[!htbp]
\centering
\caption{Fieldwise errors for the Taylor--Green vortex, reported as mean $\pm$ sample standard deviation over five repetitions. Here $E_{\mathrm{rel}}$ and $E_\infty$ denote the relative $L^2$ and maximum absolute errors of the corresponding field.}
\label{tab:taylor-green-comparison}
\begingroup
\small
\setlength{\tabcolsep}{5pt}
\renewcommand{\arraystretch}{1.2}
\par\noindent\textbf{(a) Horizontal velocity $u$}\par\smallskip
\begin{tabular}{@{}lcc@{}}
\toprule
Method & $E_{\mathrm{rel}}$ & $E_\infty$ \\
\midrule
PINN & $1.20\!\times\!10^{-3}\!\pm\!1.01\!\times\!10^{-4}$ & $2.43\!\times\!10^{-3}\!\pm\!3.29\!\times\!10^{-4}$ \\
PINN--SR & $1.49\!\times\!10^{-3}\!\pm\!1.61\!\times\!10^{-3}$ & $3.76\!\times\!10^{-3}\!\pm\!4.93\!\times\!10^{-3}$ \\
PISN & $1.45\!\times\!10^{-3}\!\pm\!2.12\!\times\!10^{-4}$ & $3.08\!\times\!10^{-3}\!\pm\!3.02\!\times\!10^{-4}$ \\
DeSyR & $4.84\!\times\!10^{-15}\!\pm\!4.87\!\times\!10^{-16}$ & $6.11\!\times\!10^{-15}\!\pm\!2.29\!\times\!10^{-17}$ \\
\bottomrule
\end{tabular}

\par\medskip
\par\noindent\textbf{(b) Vertical velocity $v$}\par\smallskip
\begin{tabular}{@{}lcc@{}}
\toprule
Method & $E_{\mathrm{rel}}$ & $E_\infty$ \\
\midrule
PINN & $1.29\!\times\!10^{-3}\!\pm\!2.61\!\times\!10^{-4}$ & $2.84\!\times\!10^{-3}\!\pm\!1.00\!\times\!10^{-3}$ \\
PINN--SR & $8.13\!\times\!10^{-4}\!\pm\!6.58\!\times\!10^{-4}$ & $1.54\!\times\!10^{-3}\!\pm\!1.47\!\times\!10^{-3}$ \\
PISN & $1.54\!\times\!10^{-3}\!\pm\!3.30\!\times\!10^{-4}$ & $2.99\!\times\!10^{-3}\!\pm\!3.05\!\times\!10^{-4}$ \\
DeSyR & $5.11\!\times\!10^{-15}\!\pm\!1.69\!\times\!10^{-21}$ & $6.20\!\times\!10^{-15}\!\pm\!3.10\!\times\!10^{-18}$ \\
\bottomrule
\end{tabular}

\par\medskip
\par\noindent\textbf{(c) Pressure $p$}\par\smallskip
\begin{tabular}{@{}lcc@{}}
\toprule
Method & $E_{\mathrm{rel}}$ & $E_\infty$ \\
\midrule
PINN & $3.44\!\times\!10^{-3}\!\pm\!4.26\!\times\!10^{-4}$ & $3.52\!\times\!10^{-3}\!\pm\!5.92\!\times\!10^{-4}$ \\
PINN--SR & $2.09\!\times\!10^{-3}\!\pm\!1.40\!\times\!10^{-3}$ & $2.47\!\times\!10^{-3}\!\pm\!1.81\!\times\!10^{-3}$ \\
PISN & $5.32\!\times\!10^{-3}\!\pm\!2.54\!\times\!10^{-3}$ & $5.37\!\times\!10^{-3}\!\pm\!3.09\!\times\!10^{-3}$ \\
DeSyR & $4.23\!\times\!10^{-15}\!\pm\!0$ & $3.46\!\times\!10^{-15}\!\pm\!0$ \\
\bottomrule
\end{tabular}

\endgroup
\end{table}
\begin{table}[!htbp]
\centering
\caption{Physical residuals for the Taylor--Green vortex, reported as mean $\pm$ sample standard deviation over five repetitions. Here $e_{\mathrm{mom},x}$, $e_{\mathrm{mom},y}$, $e_{\mathrm{div}}$ and $e_c$ denote the RMS residuals of the two momentum equations, the continuity equation and the aggregate constraints.}
\label{tab:taylor-green-physics}
\begingroup
\small
\setlength{\tabcolsep}{5pt}
\renewcommand{\arraystretch}{1.2}
\par\noindent\textbf{(a) Momentum residuals}\par\smallskip
\begin{tabular}{@{}lcc@{}}
\toprule
Method & $e_{\mathrm{mom},x}$ & $e_{\mathrm{mom},y}$ \\
\midrule
PINN & $4.78\!\times\!10^{-3}\!\pm\!2.87\!\times\!10^{-4}$ & $4.60\!\times\!10^{-3}\!\pm\!4.90\!\times\!10^{-4}$ \\
PINN--SR & $7.03\!\times\!10^{-3}\!\pm\!9.05\!\times\!10^{-3}$ & $3.99\!\times\!10^{-3}\!\pm\!2.19\!\times\!10^{-3}$ \\
PISN & $6.61\!\times\!10^{-3}\!\pm\!2.08\!\times\!10^{-3}$ & $7.90\!\times\!10^{-3}\!\pm\!3.76\!\times\!10^{-3}$ \\
DeSyR & $9.94\!\times\!10^{-15}\!\pm\!1.36\!\times\!10^{-15}$ & $1.05\!\times\!10^{-14}\!\pm\!3.87\!\times\!10^{-17}$ \\
\bottomrule
\end{tabular}

\par\medskip
\par\noindent\textbf{(b) Continuity and constraint residuals}\par\smallskip
\begin{tabular}{@{}lcc@{}}
\toprule
Method & $e_{\mathrm{div}}$ & $e_c$ \\
\midrule
PINN & $3.73\!\times\!10^{-3}\!\pm\!3.62\!\times\!10^{-4}$ & $3.03\!\times\!10^{-4}\!\pm\!3.91\!\times\!10^{-5}$ \\
PINN--SR & $2.20\!\times\!10^{-3}\!\pm\!2.04\!\times\!10^{-3}$ & $8.08\!\times\!10^{-4}\!\pm\!7.27\!\times\!10^{-4}$ \\
PISN & $3.89\!\times\!10^{-3}\!\pm\!6.68\!\times\!10^{-4}$ & $3.73\!\times\!10^{-4}\!\pm\!1.20\!\times\!10^{-4}$ \\
DeSyR & $1.15\!\times\!10^{-15}\!\pm\!2.19\!\times\!10^{-15}$ & $2.21\!\times\!10^{-15}\!\pm\!1.91\!\times\!10^{-17}$ \\
\bottomrule
\end{tabular}

\endgroup
\end{table}

\begin{samepage}
The representative DeSyR fields are
\begin{equation*}
\begin{aligned}
\widehat u(x,y,t)&=-e^{-0.19739209t}
\sin(3.1415927y)\cos(3.1415927x),\\[6pt]
\widehat v(x,y,t)&=e^{-0.19739209t}
\sin(3.1415927x)\cos(3.1415927y),\\[6pt]
\widehat p(x,y,t)&=-0.25\bigl[\cos(6.2831853x)
+\cos(6.2831853y)\bigr]e^{-0.39478418t},
\end{aligned}
\end{equation*}
with coefficients rounded for display. The recovered velocity fields
share the same temporal decay rate, while the pressure has twice the
spatial frequency and twice the temporal decay rate. The expressions
therefore reproduce both the spatial and temporal relations of the
reference vortex.
\par\end{samepage}

PINN--SR also recovers exponential factors for the velocity fields, with
separately estimated decay rates. Its pressure expression instead
represents the temporal variation by a term affine in $\sin t$ over the
sampled interval. The comparison therefore distinguishes recovery of the
pressure's exponential time dependence from approximation of its values
over that interval.

Figures~\ref{fig:taylor-green-mid-comparison} and
\ref{fig:taylor-green-comparison} show the fieldwise errors at $t=0.5$
and $t=1$. PINN--SR exhibits distinct regular spatial patterns in the
velocity and pressure errors at both times. These snapshots complement
the expression-level comparison of the velocity decay factors and the
pressure's temporal representation. DeSyR has much smaller errors in all
three fields at both displayed times, while
Tables~\ref{tab:taylor-green-comparison} and~\ref{tab:taylor-green-physics}
report field errors and coupled residuals over the full space--time test set.
\begin{figure}[htbp]
\centering
\includegraphics[width=\linewidth]{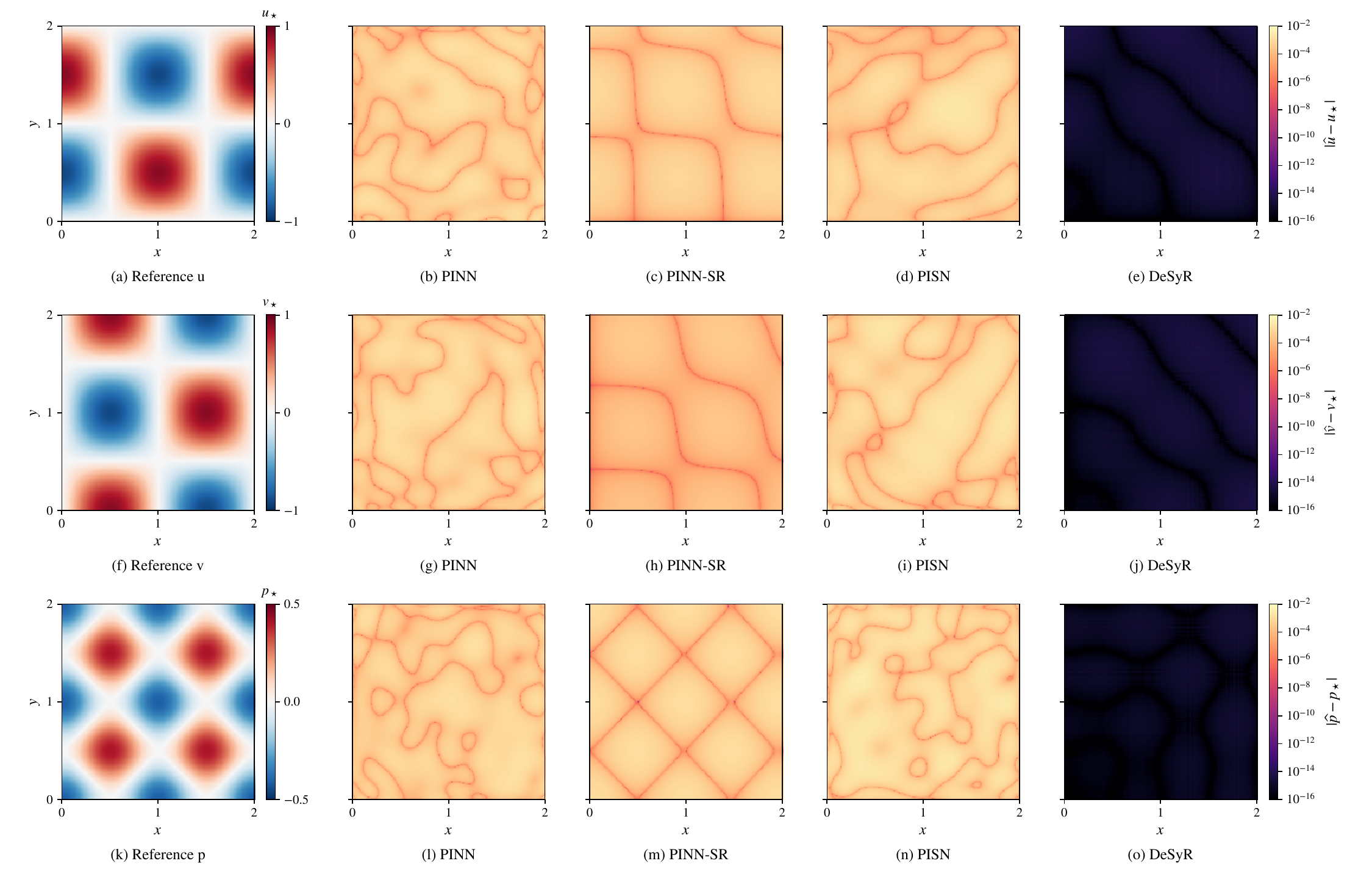}
\caption{Representative fields and absolute errors for the Taylor--Green
vortex at $t=0.5$. Rows show $u$, $v$ and $p$; columns show the reference
field and absolute errors for PINN, PINN--SR, PISN and DeSyR}
\label{fig:taylor-green-mid-comparison}
\end{figure}
\begin{figure}[htbp]
\centering
\includegraphics[width=\linewidth]{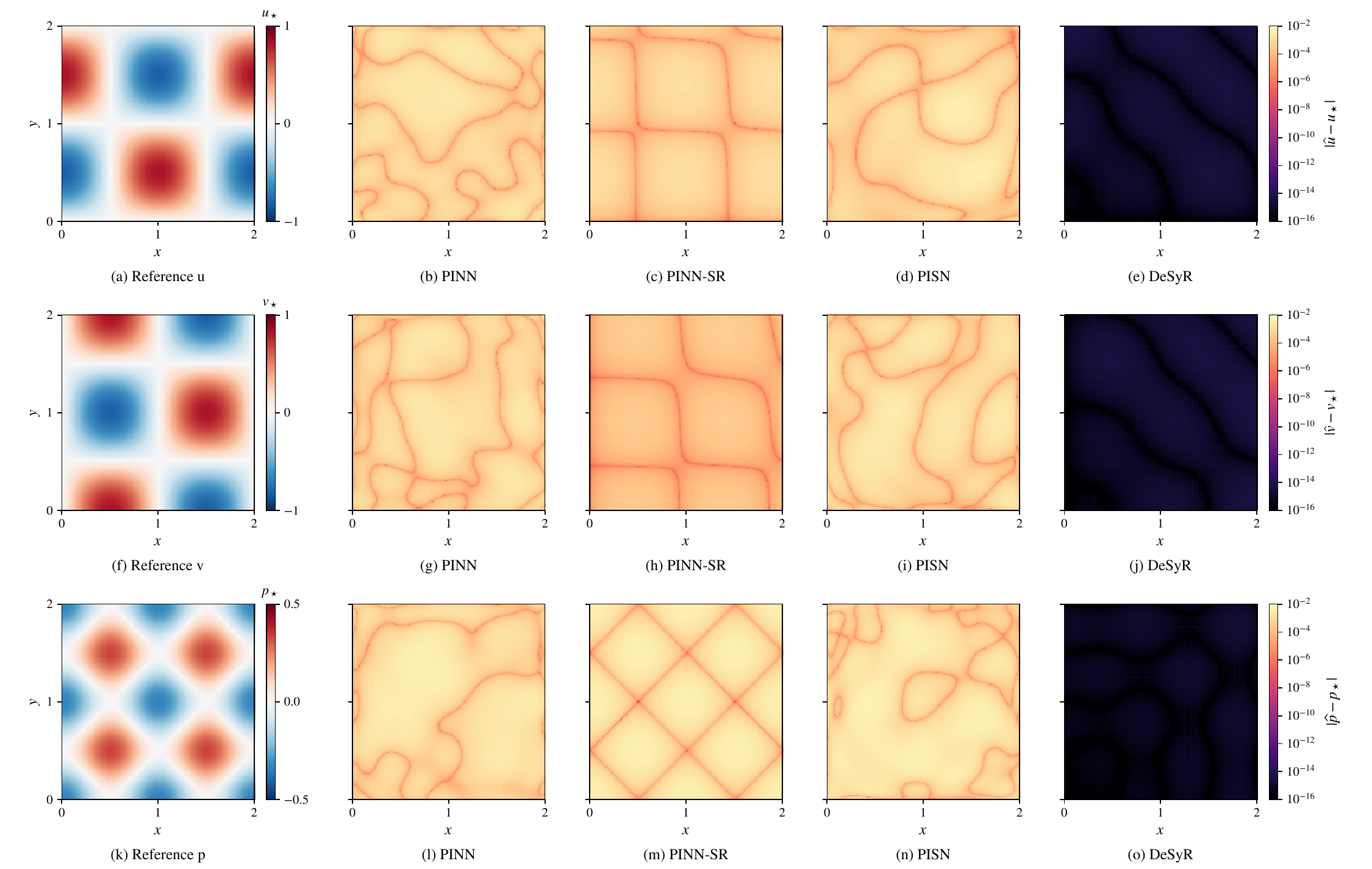}
\caption{Representative fields and absolute errors for the Taylor--Green
vortex at $t=1$. Rows show $u$, $v$ and $p$; columns show the reference
field and absolute errors for PINN, PINN--SR, PISN and DeSyR}
\label{fig:taylor-green-comparison}
\end{figure}

\subsection{Fixed-candidate comparison of coefficient-estimation objectives}
\label{sec:ablation-sensitivity}

\begin{samepage}
To examine the role of the coefficient-estimation objective, we compare
teacher fitting, mixed fitting and physics-only refinement within the same
Klein--Gordon expression family. This family exactly represents the
reference solution and has a single adjustable frequency, so coefficient
accuracy can be evaluated directly while the expression structure remains
fixed. The comparison complements the local nonlinear analysis in
Proposition~\ref{prop:nonlinear-teacher-error}.
\par\end{samepage}

\begin{samepage}
The family identified by symbolic search is
\begin{equation}
s(x,t;\omega)=x\cos(\omega t)+x^3t^3,
\label{eq:ablation-kg-family}
\end{equation}
where only $\omega$ is adjustable and $\omega_\star=5\pi$ gives the
reference solution.
\par\end{samepage}

We compare teacher fitting $\mathcal{J}_T$,
physics-only refinement $\mathcal{J}_P$, and the mixed objective
\begin{equation}
\begin{aligned}
\widetilde{\mathcal{J}}_\lambda&=\mathcal{J}_T+\lambda\mathcal{J}_P,\\
\lambda&\in\{10^{-12},10^{-10},10^{-8},10^{-6},10^{-4},10^{-3},10^{-2},10^{-1},1\}.
\end{aligned}
\label{eq:ablation-mixing-weights}
\end{equation}
Here $\mathcal{J}_T$ and $\mathcal{J}_P$ are defined in
\eqref{eq:theory-comparison-objectives}, and $\lambda=0$ gives teacher
fitting. For $\lambda>0$, $\widetilde{\mathcal{J}}_\lambda
=\lambda\mathcal{J}_\eta$ with $\eta=1/\lambda$, linking the experimental
weights to the theoretical formulation in \eqref{eq:theory-mixed-objective}.
Increasing $\lambda$ shifts the emphasis from teacher fitting toward
physical constraints. The teacher term uses mean-squared normalization,
while the physical term retains the residual sums and constraint weight
of the main experiment; the weights are interpreted on these respective
scales.

For each of the five teachers, the comparison uses 1,000 teacher
predictions and physical residuals at 4,000 interior, 2,000 boundary and
2,000 initial points, including initial-value and initial-derivative
constraints. Data and residual sampling points are fixed across
objectives. All objectives use double precision, the same four
initializations, and identical solver and stopping settings. The
reference frequency and interior reference values are used only for
evaluation.

We report the relative frequency error
$E_\omega=|\widehat\omega-\omega_\star|/|\omega_\star|$,
together with the solution errors and physical residuals defined in
Section~\ref{sec:experimental-setup}. All objectives are evaluated using
coefficients rounded to 15 significant digits, consistent with
Table~\ref{tab:klein-gordon-comparison}. The relative frequency error
$E_\omega$ is calculated from these output coefficients using 60-digit
arithmetic.

Within this exactly representable candidate family, teacher fitting yields
mean relative frequency and solution errors of $9.83\times10^{-5}$ and
$8.47\times10^{-4}$, respectively
(Table~\ref{tab:kg-objective-comparison}). With $\lambda=10^{-12}$,
the mixed estimates remain close to the teacher-fitted values. Increasing
the physical weight progressively reduces both errors: the mean relative
solution error is $2.21\times10^{-4}$ at $\lambda=10^{-8}$ and
$2.98\times10^{-6}$ at $\lambda=10^{-6}$. Over
$\lambda=10^{-4}$ to $1$, each tenfold increase in physical weight
reduces the mean frequency and solution errors by approximately one
order of magnitude. 

\begin{samepage}
At $\lambda=1$, the mean relative solution error is
$2.98\times10^{-12}$ and the mean PDE residual RMS is
$3.26\times10^{-10}$. Physics-only refinement gives smaller corresponding
values of $1.85\times10^{-14}$ and $2.02\times10^{-12}$.
\par\end{samepage}

\begin{table}[htbp]
\centering
\caption{Fixed-candidate coefficient estimation for the Klein--Gordon problem: mean $\pm$ sample standard deviation over five teachers. Here $E_\omega$ is the relative frequency error, $E_{\mathrm{rel}}$ the relative $L^2$ error, and $e_r$ and $e_c$ the PDE and aggregate constraint residual RMS values.}
\label{tab:kg-objective-comparison}
\begingroup
\small
\setlength{\tabcolsep}{5pt}
\renewcommand{\arraystretch}{1.2}
\par\noindent\textbf{(a) Coefficient and solution errors}\par\smallskip
\begin{tabular}{@{}lcc@{}}
\toprule
Objective & $E_\omega$ & $E_{\mathrm{rel}}$ \\
\midrule
Teacher fitting & $9.83\!\times\!10^{-5}\!\pm\!6.59\!\times\!10^{-5}$ & $8.47\!\times\!10^{-4}\!\pm\!5.68\!\times\!10^{-4}$ \\
Mixed, $\lambda=10^{-12}$ & $9.83\!\times\!10^{-5}\!\pm\!6.59\!\times\!10^{-5}$ & $8.46\!\times\!10^{-4}\!\pm\!5.68\!\times\!10^{-4}$ \\
Mixed, $\lambda=10^{-10}$ & $9.56\!\times\!10^{-5}\!\pm\!6.41\!\times\!10^{-5}$ & $8.23\!\times\!10^{-4}\!\pm\!5.52\!\times\!10^{-4}$ \\
Mixed, $\lambda=10^{-8}$ & $2.57\!\times\!10^{-5}\!\pm\!1.72\!\times\!10^{-5}$ & $2.21\!\times\!10^{-4}\!\pm\!1.48\!\times\!10^{-4}$ \\
Mixed, $\lambda=10^{-6}$ & $3.46\!\times\!10^{-7}\!\pm\!2.32\!\times\!10^{-7}$ & $2.98\!\times\!10^{-6}\!\pm\!2.00\!\times\!10^{-6}$ \\
Mixed, $\lambda=10^{-4}$ & $3.47\!\times\!10^{-9}\!\pm\!2.33\!\times\!10^{-9}$ & $2.99\!\times\!10^{-8}\!\pm\!2.01\!\times\!10^{-8}$ \\
Mixed, $\lambda=10^{-3}$ & $3.47\!\times\!10^{-10}\!\pm\!2.33\!\times\!10^{-10}$ & $2.99\!\times\!10^{-9}\!\pm\!2.01\!\times\!10^{-9}$ \\
Mixed, $\lambda=10^{-2}$ & $3.47\!\times\!10^{-11}\!\pm\!2.33\!\times\!10^{-11}$ & $2.99\!\times\!10^{-10}\!\pm\!2.01\!\times\!10^{-10}$ \\
Mixed, $\lambda=10^{-1}$ & $3.47\!\times\!10^{-12}\!\pm\!2.33\!\times\!10^{-12}$ & $2.99\!\times\!10^{-11}\!\pm\!2.01\!\times\!10^{-11}$ \\
Mixed, $\lambda=1$ & $3.46\!\times\!10^{-13}\!\pm\!2.33\!\times\!10^{-13}$ & $2.98\!\times\!10^{-12}\!\pm\!2.01\!\times\!10^{-12}$ \\
Physics-only refinement & $2.11\!\times\!10^{-15}\!\pm\!0$ & $1.85\!\times\!10^{-14}\!\pm\!0$ \\
\bottomrule
\end{tabular}

\par\medskip
\par\noindent\textbf{(b) Physical residuals}\par\smallskip
\begin{tabular}{@{}lcc@{}}
\toprule
Objective & $e_r$ & $e_c$ \\
\midrule
Teacher fitting & $9.25\!\times\!10^{-2}\!\pm\!6.21\!\times\!10^{-2}$ & $3.15\!\times\!10^{-4}\!\pm\!2.11\!\times\!10^{-4}$ \\
Mixed, $\lambda=10^{-12}$ & $9.25\!\times\!10^{-2}\!\pm\!6.20\!\times\!10^{-2}$ & $3.15\!\times\!10^{-4}\!\pm\!2.11\!\times\!10^{-4}$ \\
Mixed, $\lambda=10^{-10}$ & $9.00\!\times\!10^{-2}\!\pm\!6.03\!\times\!10^{-2}$ & $3.06\!\times\!10^{-4}\!\pm\!2.05\!\times\!10^{-4}$ \\
Mixed, $\lambda=10^{-8}$ & $2.42\!\times\!10^{-2}\!\pm\!1.62\!\times\!10^{-2}$ & $8.22\!\times\!10^{-5}\!\pm\!5.51\!\times\!10^{-5}$ \\
Mixed, $\lambda=10^{-6}$ & $3.26\!\times\!10^{-4}\!\pm\!2.19\!\times\!10^{-4}$ & $1.11\!\times\!10^{-6}\!\pm\!7.44\!\times\!10^{-7}$ \\
Mixed, $\lambda=10^{-4}$ & $3.27\!\times\!10^{-6}\!\pm\!2.19\!\times\!10^{-6}$ & $1.11\!\times\!10^{-8}\!\pm\!7.47\!\times\!10^{-9}$ \\
Mixed, $\lambda=10^{-3}$ & $3.27\!\times\!10^{-7}\!\pm\!2.19\!\times\!10^{-7}$ & $1.11\!\times\!10^{-9}\!\pm\!7.47\!\times\!10^{-10}$ \\
Mixed, $\lambda=10^{-2}$ & $3.27\!\times\!10^{-8}\!\pm\!2.19\!\times\!10^{-8}$ & $1.11\!\times\!10^{-10}\!\pm\!7.46\!\times\!10^{-11}$ \\
Mixed, $\lambda=10^{-1}$ & $3.27\!\times\!10^{-9}\!\pm\!2.19\!\times\!10^{-9}$ & $1.11\!\times\!10^{-11}\!\pm\!7.46\!\times\!10^{-12}$ \\
Mixed, $\lambda=1$ & $3.26\!\times\!10^{-10}\!\pm\!2.20\!\times\!10^{-10}$ & $1.11\!\times\!10^{-12}\!\pm\!7.48\!\times\!10^{-13}$ \\
Physics-only refinement & $2.02\!\times\!10^{-12}\!\pm\!0$ & $6.89\!\times\!10^{-15}\!\pm\!0$ \\
\bottomrule
\end{tabular}
\endgroup
\end{table}

Figure~\ref{fig:kg-objective-comparison} shows a teacher-dominated regime
at small physical weights and an approximately inverse-weight decrease
in frequency and solution errors at larger weights. In the theoretical
parameterization $\eta=1/\lambda$, this trend is consistent with the
reduced contribution of the teacher-error term as the physical objective
becomes dominant.

\begin{samepage}
With the expression structure held fixed, physics-only refinement yields
the smallest frequency and solution errors among the compared objectives.
The result demonstrates the value of determining coefficients from the
physical problem after neural guidance has identified a suitable
expression family. Together with the teacher-error analysis, this
comparison supports DeSyR's separation of roles: teacher predictions guide
structure discovery and initialization, while the governing equations and
constraints determine the refined coefficients.
\par\end{samepage}

\begin{figure}[htbp]
\centering
\includegraphics[width=\linewidth]{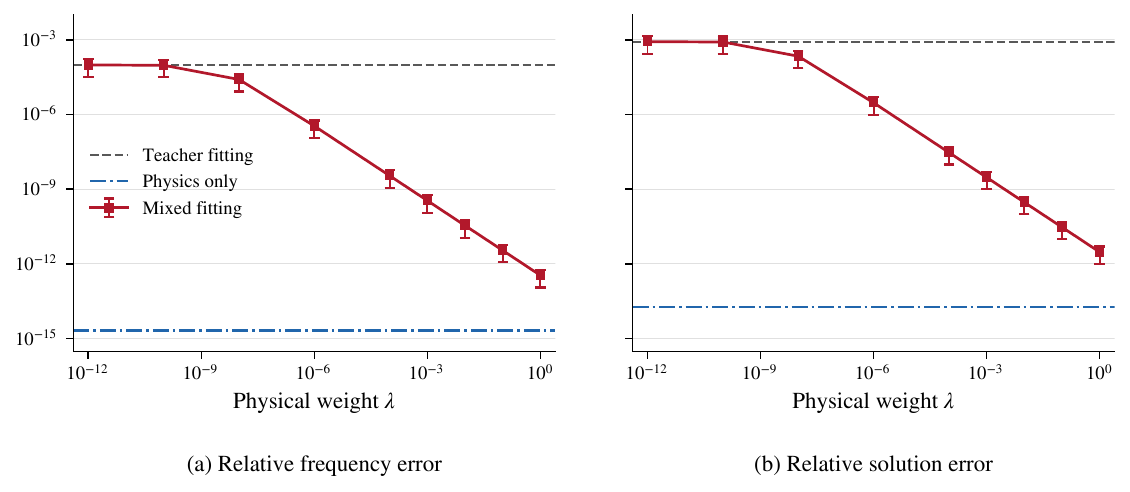}
\caption{Fixed-candidate Klein--Gordon comparison: (a) relative frequency
error and (b) relative solution error versus the physical weight $\lambda$.
Markers and error bars show means and one sample standard deviation over
five teachers. Horizontal lines show the mean errors of teacher fitting
and physics-only refinement}
\label{fig:kg-objective-comparison}
\end{figure}

\FloatBarrier

\section{Conclusions}
\label{sec:conclusion}

We presented DeSyR, a framework for recovering explicit symbolic solutions
of PDEs through PINN-guided candidate discovery and physics-based coefficient
refinement. Multiple candidate structures are retained for physical
assessment, and the final expression is selected by minimum complexity
within a physically near-optimal set. For coupled systems, joint refinement
incorporates the complete equations and constraints to determine the
coefficients of the component fields. This design uses neural
approximations to guide expression discovery and physical relations to
determine the final symbolic solution.

The theoretical analysis characterizes the role of the coefficient-estimation
objective within a fixed parameterization. Affine error identities and local
nonlinear perturbation results describe how teacher errors enter coefficients
estimated by teacher-fitting and mixed objectives. Under the stated
representability, regularity and identifiability conditions, physical-residual
bounds control coefficient and solution errors, with local results for
nonlinear parameterizations. Together, these results explain the distinction
between inheriting approximation errors through teacher fitting and
determining coefficients through physical constraints.

Numerical comparisons with PINNs, physics-informed symbolic networks
(PISN) and teacher-fitted PINN--SR on seven PDE problems demonstrate accurate
symbolic recovery across scalar evolution equations, multidimensional elliptic
problems and coupled incompressible flows. Across these benchmarks, DeSyR
yields smaller reported mean solution errors and physical residuals than the
PINN teachers, PISN and teacher-fitted PINN--SR. The coupled-flow results
additionally reproduce spatial and temporal relations among velocity and
pressure fields. In the fixed-candidate comparison, physics-only refinement
yields the smallest frequency and solution errors among the tested
objectives. Together, these results support assigning distinct roles to
neural predictions and physical constraints: the former guide candidate
discovery and initialization, while the latter determine the refined
coefficients.

The present study evaluates the framework on benchmarks covering scalar
evolution equations, multidimensional elliptic problems and coupled
incompressible flows, all with known reference solutions. Further directions include
extending candidate discovery to more complex solution structures and
coupled systems, and developing independent verification procedures for
problems without analytical references.

\section*{Declarations}

\subsection*{Funding}
The authors gratefully acknowledge financial support from the National Natural Science Foundation of China (Grant No.~12371434) and the Sichuan Provincial Department of Science and Technology (Project No.~2026NSFSC0138).

\subsection*{Competing interests}
The authors have no relevant financial or non-financial interests to disclose.

\subsection*{Data availability}
The data supporting the findings of this study are available from the
corresponding author upon reasonable request.

\bibliographystyle{unsrtnat}
\begingroup
\raggedright
\bibliography{references}
\endgroup

\end{document}